\documentclass{article}

\PassOptionsToPackage{numbers,sort&compress}{natbib}
\usepackage[main, preprint]{neurips_2026}
\usepackage{amsthm}       % 用于定理环境
\usepackage{algorithm}   % 提供浮动体环境
\usepackage{algorithmic} % 提供伪代码命令
\usepackage{graphicx} % 插入图片的核心包[reference:0]
\usepackage{subcaption} % 可选，如果你需要并排多张图[reference:1]
\usepackage{amsmath, amssymb} % 数学公式必备[reference:2]

\newtheorem{theorem}{Theorem}

\newtheorem{lemma}{Lemma}

\newtheorem{assumption}{Assumption}

\usepackage[utf8]{inputenc} % allow utf-8 input
\usepackage[T1]{fontenc}    % use 8-bit T1 fonts
\usepackage{hyperref}       % hyperlinks
\usepackage{url}            % simple URL typesetting
\usepackage{booktabs}       % professional-quality tables
\usepackage{amsfonts}       % blackboard math symbols
\usepackage{nicefrac}       % compact symbols for 1/2, etc.
\usepackage{microtype}      % microtypography
\usepackage{xcolor}         % colors
\usepackage{tabularx}
\usepackage{booktabs}
\usepackage{multirow}
\usepackage{makecell}
\usepackage{array}

\newcolumntype{C}[1]{>{\centering\arraybackslash}m{#1}}

\newcommand{\best}[1]{\textbf{#1}}
\newcommand{\second}[1]{\smash{\underline{#1}}}

\newcolumntype{C}[1]{>{\centering\arraybackslash}m{#1}}

\title{Momentum as Residual-Driven Multiplier Correction for Deep Learning Optimization}

\author{
  Zhixin Ren$^{1}$$^{\ast}$, Yao Lyu$^{2}$\,$^{3}$$^{\ast}$, Congrong Li$^{4}$, Liping Zhang$^{1}$, Shengbo Eben Li$^{2}$\,$^{4}$\textdagger  \\
  $^{1}$ Department of Mathmatical Science, Tsinghua University\\
  $^{2}$School of Vehicle and Mobility, Tsinghua University\\
  $^{3}$SunRisingAI Lab\,\,
  $^{4}$College of AI, Tsinghua University\\
  ${\ast}$ Equal contribution\,\,\, \textdagger Corresponding author
}
\begin{document}

\maketitle

\begin{abstract}
Momentum-based optimizers are widely used in modern deep learning, yet the relations among momentum recursion, update geometry, and acceleration remain only partially understood. We develop an \textbf{A}DMM-\textbf{I}nspired \textbf{M}omentum (AIM) framework based on residual-penalty variable splitting, which interprets momentum as a multiplier-like correction driven by the splitting residual. AIM recovers the exponential moving average of gradients from an ADMM-style multiplier update and separates two mechanisms that are usually intertwined in practical optimizers: the residual penalty determines the update geometry, whereas the approximation of the objective-related subproblem determines the acceleration form. Building on AIM, we propose \textbf{R}elativistic \textbf{A}daptive gradient \textbf{D}escent with \textbf{A}ccelerated \textbf{R}esidual (RADAR), which combines relativistic adaptive geometry, decoupled residual correction, and gradient-difference momentum filtering to improve the update direction and momentum estimation. We establish stochastic convergence through a variance-perturbed Lyapunov drift analysis. Experiments on supervised vision learning, language modeling, and reinforcement learning show that RADAR achieves consistent improvements over strong adaptive optimizer baselines.
\end{abstract}

\section{Introduction}

Optimization algorithms are a fundamental component of modern deep learning. Among widely used optimizers, SGD with momentum~\cite{polyak1964some} remains a strong baseline due to its generalization ability, while Adam-type methods~\cite{kingma2015adam,loshchilov2019decoupled} have become standard choices because of their robustness and adaptive coordinate-wise scaling. Despite their different update rules, many successful optimizers share several recurring ingredients: a momentum estimator that aggregates gradient information, an update geometry that determines the descent direction, and an acceleration-style correction that refines the parameter update. Understanding how these ingredients interact is important for both explaining existing optimizers and designing new ones.

Momentum is commonly viewed as an exponential moving average of gradients, or equivalently as a low-pass filter that suppresses stochastic noise. While this interpretation explains its variance-reduction effect, it does not clarify why momentum takes its specific recursive form or how it should interact with parameter updates. Acceleration methods offer a complementary view: Nesterov momentum~\cite{nesterov1983method} introduces a forward-looking correction to heavy-ball momentum, while NAdam~\cite{dozat2016incorporating} combines this idea with adaptive preconditioning. Yet these methods are typically presented through specific update rules, leaving the underlying correction principle implicit.

A complementary line of work explains optimizer updates from the viewpoint of geometry and dynamical constraints. Norm-based analyses relate different optimizers to steepest descent under different geometries~\cite{bernstein2024old}: Adam-type methods can be understood through coordinate-wise adaptive geometry, while matrix-norm optimizers such as Muon~\cite{jordan2024muon} are associated with matrix-level steepest descent. In another direction, RAD~\cite{lyu2024conformal} introduces relativistic adaptive geometry and conformal symplectic structure into Adam-type updates, aiming to improve long-term training stability through speed-limiting and structure-preserving considerations. These views clarify important aspects of update geometry and training dynamics, but they do not fully explain how momentum estimation, geometric preconditioning, and acceleration-style correction should be coupled in a single optimizer. As a result, momentum is often treated as a filtering heuristic, geometry as a preconditioner, and acceleration as an additional correction rule. A framework that separates these roles while keeping their interactions explicit would therefore provide both a unified interpretation of existing optimizers and a systematic route for designing new momentum-based algorithms.

In this paper, we develop an \textbf{A}DMM-\textbf{I}nspired
\textbf{M}omentum (AIM) framework through residual-penalty variable
splitting. Unlike approaches in which the momentum mechanism and the
parameter-update rule are designed largely as separate components, AIM
derives both from a common ADMM optimization formulation. By introducing an
auxiliary descent variable and coupling it with the network parameter
through a splitting residual, AIM interprets momentum as a multiplier-like
variable updated by a relaxed residual-driven correction.

AIM further separates three components that are often intertwined in
practical optimizers. The residual penalty determines the update geometry;
the multiplier or momentum-filtering rule determines how gradient
information is accumulated; and the approximation of the
$\theta$-subproblem determines how the tentative descent point is corrected.
In particular, different approximations of the $\theta$-subproblem provide
an optimization-based interpretation of heavy-ball and Nesterov-type
updates. Combined with different residual geometries, this formulation
organizes a broad range of Euclidean, adaptive, and matrix-based optimizers
within a unified design space and provides a systematic route for
constructing new optimizer variants.

Building on this framework, we propose
\textbf{R}elativistic \textbf{A}daptive gradient \textbf{D}escent with
\textbf{A}ccelerated \textbf{R}esidual (RADAR). RADAR retains the
relativistic adaptive preconditioning structure of RAD while introducing
an AIM-derived gradient-difference momentum correction and a decoupled
residual-correction step obtained from a one-step fixed-point approximation
of the $\theta$-subproblem. These modifications incorporate Nesterov-type
acceleration into the RAD update without changing its underlying
preconditioning structure.

Our contributions are summarized as follows:
\begin{itemize}
    \item We propose AIM, an ADMM-inspired residual-penalty splitting
    framework that jointly connects momentum mechanisms and
    parameter-update rules. AIM interprets momentum as a multiplier-like
    variable and derives the exponential moving-average recursion from a
    relaxed residual-driven multiplier update. The corresponding full
    correction enforces the KKT stationarity condition with respect to
    $\theta$, providing an optimization-based interpretation of momentum
    rather than treating it as an independently postulated heuristic.

    \item We use AIM to distinguish three optimizer design components:
    residual geometry, multiplier or momentum filtering, and approximation
    of the objective-related $\theta$-subproblem. Different residual
    penalties induce Euclidean, adaptive diagonal, and matrix-level update
    geometries, whereas different approximations of the
    $\theta$-subproblem distinguish direct momentum updates from
    Nesterov-type residual corrections. This formulation provides a unified
    interpretation of multiple optimizer families and a systematic design
    space for constructing new variants.

    \item We propose RADAR, an AIM-derived extension of the
    structure-preserving optimizer RAD. RADAR combines relativistic
    adaptive preconditioning, an AIM-derived gradient-difference momentum
    correction, and a decoupled residual correction obtained from an
    approximate solution of the $\theta$-subproblem. We establish its
    stochastic convergence through a variance-perturbed Lyapunov drift
    analysis, and experiments on supervised vision learning, language
    modeling, and reinforcement learning validate its effectiveness.
\end{itemize}

\section{Preliminaries}
\label{sec:preliminaries}

We briefly recall the \textbf{A}lternating \textbf{D}irection \textbf{M}ethod of \textbf{M}ultipliers (ADMM)~\cite{doi:10.1137/1.9781611974997} for the linearly constrained problem
\begin{equation}
\min_{x,z}\ \phi(x)+\varphi(z),
\qquad
\mathrm{s.t.}\ Ax+Bz=c,
\label{eq:admm_problem}
\end{equation}
where $\phi$ and $\varphi$ are proper closed convex functions. Its augmented Lagrangian is
\begin{equation}
L_\rho(x,z,\lambda)
=
\phi(x)+\varphi(z)
+
\langle \lambda,Ax+Bz-c\rangle
+
\frac{\rho}{2}\|Ax+Bz-c\|^2 ,
\label{eq:admm_lagrangian}
\end{equation}
where $\lambda$ is the Lagrange multiplier and $\rho>0$ is the penalty parameter. The last term is the standard quadratic penalty on the primal residual. ADMM alternates between two primal minimization steps and a multiplier update:
\begin{align}
x_{k+1}&=\arg\min_x L_\rho(x,z_k,\lambda_k),\\
z_{k+1}&=\arg\min_z L_\rho(x_{k+1},z,\lambda_k),
\label{eq:admm_z_update}\\
\lambda_{k+1}&=\lambda_k+\tau\rho(Ax_{k+1}+Bz_{k+1}-c),
\label{eq:admm_multiplier_update}
\end{align}

where $\tau>0$ is the stepsize of the multiplier update. In the standard case $\tau=1$, combining the multiplier update with the optimality condition of \eqref{eq:admm_z_update} yields the KKT stationarity condition of Problem~\eqref{eq:admm_problem} with respect to $z$. This classical connection between multiplier updates and KKT stationarity provides the key motivation for the multiplier construction in AIM. As established in Appendix~\ref{app:multiplier_update}, the full residual-driven correction in AIM enforces the analogous KKT stationarity condition with respect to $\theta$, while a relaxed correction naturally reduces to the standard exponential moving-average momentum recursion. This connection provides a principled interpretation of momentum as a relaxed multiplier update within the AIM framework.  

\section{An ADMM-Inspired Framework for Momentum-Based Optimization}
\label{sec:framework}

Motivated by the residual-driven multiplier update in classical ADMM, we introduce a residual-penalty variable-splitting framework for momentum-based neural network optimization. The framework introduces an auxiliary descent variable, couples it with the network parameter through a splitting residual, and interprets momentum as a multiplier-like correction driven by this residual. It also separates two mechanisms that are usually intertwined in practical optimizers: the residual penalty determines the update geometry, while the approximation of the objective-related subproblem determines the acceleration form.

\subsection{Residual-Penalty Splitting and Multiplier Momentum}
\label{subsec:residual_penalty_splitting}

Consider the empirical risk minimization problem as follows:
\begin{equation}
\min_{\theta\in\mathbb{R}^d}
\mathcal{L}(\theta)
:=
\frac{1}{N}\sum_{i=1}^{N} l(\theta,\xi_i)
 ,
\label{eq:erm_problem}
\end{equation}
where $\mathcal{L}:\mathbb{R}^d\to\mathbb{R}$ is the objective function, $l(\theta,\xi)$ is the sample-wise loss, and $\theta\in\mathbb{R}^d$ is the trainable parameter. By introducing an auxiliary descent variable $y$ and the equality constraint $y=\theta$, we rewrite \eqref{eq:erm_problem} as
\begin{equation}
\min_{\theta,y}\ \mathcal{L}(\theta),
\qquad
\mathrm{s.t.}\ y-\theta=0 .
\label{eq:splitting_problem}
\end{equation}
Here, $\theta$ carries the original objective, while $y$ is an auxiliary copy introduced to isolate the descent step. During alternating updates, $y$ and $\theta$ may differ, and the discrepancy $y-\theta$ is the splitting residual to be controlled.

Motivated by the augmented Lagrangian of ADMM, we define the residual-penalty augmented Lagrangian as follows:
\begin{equation}
\overline L_\rho(\theta,y,m)
=
\mathcal{L}(\theta)
+
\langle m,y-\theta\rangle
+
\frac{\rho}{2}\psi(y-\theta),
\label{eq:residual_lagrangian}
\end{equation}
where $m$ is a multiplier-like variable, $\rho>0$, and $\psi$ is a proper closed convex residual penalty. The classical quadratic penalty corresponds to $\psi(r)=\|r\|^2$, which measures the residual in the Euclidean geometry. Allowing a general convex $\psi$ extends the residual geometry beyond the Euclidean quadratic case and provides a flexible way to describe residual-driven correction.

We refer to this residual-penalty splitting construction as the \textbf{A}DMM-\textbf{I}nspired \textbf{M}omentum (AIM) framework, summarized in Algorithm~\ref{alg:admm_inspired_momentum}. The $y$-subproblem contains no objective term and therefore specifies the descent geometry induced by $\psi$, while the $\theta$-subproblem contains $\mathcal{L}(\theta)$ and determines how objective-gradient information corrects the tentative descent point $y_{k+1}$. The two penalty parameters $\rho_1$ and $\rho_2$ scale the geometry-producing step and the objective-correction step separately. The multiplier update is a relaxed ADMM-style correction applied to the
splitting residual using the generalized residual subgradient induced by
\(\psi\). As shown in Appendix~\ref{app:multiplier_update}, the corresponding
full residual correction exactly enforces the KKT stationarity condition
with respect to \(\theta\), while the relaxed correction recovers the
standard exponential moving-average momentum recursion. Thus, momentum in
AIM can be interpreted as a relaxed multiplier update that tracks the
KKT-consistent relation between the multiplier and the objective gradient.

\begin{algorithm}[t]
\caption{ADMM-Inspired Momentum (AIM) Framework}
\label{alg:admm_inspired_momentum}
\small
\begin{algorithmic}[1]
\REQUIRE $\theta_0=y_0$, $m_0$, $\beta_1\in[0,1)$, $\rho_1,\rho_2>0$, $T$
\FOR{$k=0,1,\ldots,T-1$}
    \STATE $y_{k+1}
    =
    \arg\min_y
    \left\{
    \langle m_k,y-\theta_k\rangle
    +
    \frac{\rho_1}{2}\psi(y-\theta_k)
    \right\}$
    \STATE $\theta_{k+1}
    \approx
    \arg\min_\theta
    \left\{
    \mathcal{L}(\theta)
    +
    \langle m_k,y_{k+1}-\theta\rangle
    +
    \frac{\rho_2}{2}\psi(y_{k+1}-\theta)
    \right\}$
    \STATE $m_{k+1}
    =
    m_k+(1-\beta_1)\frac{\rho_2}{2}d_{k+1},
    \quad
    d_{k+1}\in\partial\psi(y_{k+1}-\theta_{k+1})$
\ENDFOR
\RETURN $\theta_T$
\end{algorithmic}
\end{algorithm}

We now show that the multiplier-like variable generated by Algorithm~\ref{alg:admm_inspired_momentum} naturally satisfies the standard momentum recursion. Consider the idealized case where the $\theta$-subproblem is solved exactly. Its first-order optimality condition gives
\begin{equation}
0
\in
\nabla \mathcal{L}(\theta_{k+1})-m_k-\frac{\rho_2}{2}d_{k+1},
\qquad
d_{k+1}\in\partial\psi(y_{k+1}-\theta_{k+1}).
\label{eq:theta_optimality}
\end{equation}
Substituting this relation into the multiplier update yields
\begin{equation}
m_{k+1}
=
m_k+(1-\beta_1)\bigl(\nabla \mathcal{L}(\theta_{k+1})-m_k\bigr)
=
\beta_1 m_k+(1-\beta_1)\nabla \mathcal{L}(\theta_{k+1}).
\label{eq:momentum_from_multiplier}
\end{equation}
Thus, the exponential moving average used in momentum is recovered from an ADMM-style multiplier update, rather than being imposed as an independent heuristic. In the usual optimizer view, \eqref{eq:momentum_from_multiplier} is a gradient averaging rule. In the proposed variable-splitting view, it is a multiplier correction driven by the splitting residual.

Specifically, for the Euclidean penalty $\psi(r)=\|r\|^2$, we have $d_{k+1}=2(y_{k+1}-\theta_{k+1})$. Combining this relation with the multiplier update and \eqref{eq:theta_optimality} gives
\begin{equation}
\begin{aligned}
m_{k+1}-m_k
&=
(1-\beta_1)\rho_2(y_{k+1}-\theta_{k+1}),\\
\rho_2(y_{k+1}-\theta_{k+1})
&=
\nabla \mathcal{L}(\theta_{k+1})-m_k .
\end{aligned}
\label{eq:residual_two_roles}
\end{equation}
Therefore, the splitting residual has two roles in the Euclidean case: it drives the ADMM-style multiplier update and, at the same time, measures the scaled mismatch between the current objective gradient and the previous momentum estimate. This gives momentum a residual-correction interpretation that complements the conventional low-pass filtering view. The gradient--momentum mismatch identified here will serve as the correction signal for the $\theta$-subproblem below.

\subsection{Nesterov Acceleration from the $\theta$-Subproblem}
\label{subsec:nesterov}

We next show how Nesterov-type acceleration arises from the approximation of the $\theta$-subproblem. The $y$-subproblem first produces a tentative descent point along the momentum direction, whereas the $\theta$-subproblem further corrects this point using objective-gradient information. In the Euclidean case, this correction is governed by the gradient--momentum mismatch in \eqref{eq:residual_two_roles}. Directly setting $\theta_{k+1}=y_{k+1}$ gives heavy-ball momentum, while keeping a first-order explicit approximation of this correction leads to a Nesterov-type update.

Specifically, for $\psi(r)=\|r\|^2$, the $y$-subproblem gives
\begin{equation}
y_{k+1}
=
\theta_k-\frac{1}{\rho_1}m_k .
\label{eq:y_update_euclidean}
\end{equation}
Taking $\rho_1=1/\eta$ yields the tentative momentum descent point $y_{k+1}=\theta_k-\eta m_k$. The exact optimality condition of the $\theta$-subproblem can be written as the implicit fixed-point equation
\begin{equation}
\theta_{k+1}
=
y_{k+1}
-
\eta(1-\beta_1) \bigl(\nabla \mathcal{L}(\theta_{k+1})-m_k\bigr),
\qquad
\rho_2=\frac{1}{\eta(1-\beta_1)}.
\label{eq:theta_exact_nesterov}
\end{equation}
This equation refines $y_{k+1}$ by a correction proportional to the gradient--momentum mismatch. Since the correction depends on the unknown $\theta_{k+1}$, the exact update is implicit. Applying a one-step fixed-point iteration initialized at $\theta_k$ gives the explicit approximation
\begin{equation}
\begin{aligned}
\theta_{k+1}
&=
y_{k+1}
-
\eta(1-\beta_1) \bigl(\nabla \mathcal{L}(\theta_k)-m_k\bigr)\\
&=
\underbrace{\theta_k-\eta m_k}_{\text{momentum descent}}
\underbrace{-\eta(1-\beta_1) \bigl(\nabla \mathcal{L}(\theta_k)-m_k\bigr)}_{\text{residual correction}} .
\end{aligned}
\label{eq:theta_nesterov_approx}
\end{equation}

\begin{theorem}
\label{thm:nesterov}
Consider Algorithm~\ref{alg:admm_inspired_momentum} with $\psi(r)=\|r\|^2$, $\rho_1=1/\eta$, and $\rho_2=1/(\eta(1-\beta_1))$. If the $\theta$-subproblem is approximated by \eqref{eq:theta_nesterov_approx}, the resulting update recovers a Nesterov-type momentum method up to a change of variables. If the simpler approximation $\theta_{k+1}=y_{k+1}$ is used instead, the correction term vanishes and the update reduces to heavy-ball momentum.
\end{theorem}

Detailed derivations for Theorem~\ref{thm:nesterov} are provided in Appendix~\ref{app:nesterov_details}. Theorem~\ref{thm:nesterov} gives an ADMM-inspired interpretation of Nesterov acceleration: acceleration does not come from modifying the multiplier update itself, but from retaining an explicit approximation of the residual correction in the $\theta$-subproblem.

\subsection{Adaptive Geometry from the $y$-Subproblem}
\label{subsec:adaptive_geometry}

While Section~\ref{subsec:nesterov} shows that the approximation of the $\theta$-subproblem determines the acceleration form, we now turn to the $y$-subproblem, which determines the update geometry. Since the $y$-subproblem contains no objective term, changing the residual penalty $\psi$ directly changes the geometry of the descent direction.

For Adam-type methods, let $v_k$ denote the second-moment estimate available at iteration $k$:
\begin{equation}
v_k=\beta_2v_{k-1}+(1-\beta_2){\nabla\mathcal{L}(\theta_k)}^2,
\qquad
Q_k=\mathrm{Diag}(\sqrt{v_k}+\epsilon),
\label{eq:adaptive_geometry_matrix}
\end{equation}
where $\beta_2$ is the second-moment coefficient. Using the weighted residual penalty $\psi(r)=\|r\|_{Q_k}^2$ changes the $y$-subproblem into the preconditioned momentum step as follows:
\begin{equation}
y_{k+1}
=
\theta_k-\eta Q_k^{-1}m_k,
\qquad
\rho_1=1/\eta .
\label{eq:adam_y_update}
\end{equation}
Thus, the second-moment estimate changes the residual geometry rather than the multiplier-like role of $m_k$. Under this adaptive geometry, the direct approximation $\theta_{k+1}=y_{k+1}$ recovers Adam without bias correction. The Nesterov-type approximation
\begin{equation}
\begin{aligned}
\theta_{k+1}
&=
y_{k+1}
-
\eta(1-\beta_1) Q_k^{-1}\bigl(\nabla \mathcal{L}(\theta_k)-m_k\bigr)\\
&=
\theta_k
-
\eta Q_k^{-1}m_k
-
\eta(1-\beta_1) Q_k^{-1}\bigl(\nabla \mathcal{L}(\theta_k)-m_k\bigr)
\end{aligned}
\label{eq:nadam_theta_update}
\end{equation}
recovers NAdam without bias correction. Therefore, Adam and NAdam share the same adaptive residual geometry, but differ in how the $\theta$-subproblem is approximated.

The same geometric view also extends beyond coordinate-wise adaptive preconditioning. Replacing the adaptive diagonal geometry with a matrix-level steepest-descent geometry yields a Muon-type update under the direct approximation of matrix-valued variables $\Theta_{k+1}=Y_{k+1}$. We provide the detailed derivation in Appendix~\ref{app:adaptive_matrix_extensions}. This example further illustrates that AIM treats Euclidean, adaptive diagonal, and matrix-norm updates as different residual geometries under the same multiplier-correction structure.

\subsection{A Unified Optimizer Design Space}
\label{subsec:aim_design_space}

The preceding analysis reveals two fundamental design axes within AIM. The
residual penalty $\psi$ determines the geometry of the tentative update,
while the approximation adopted for the $\theta$-subproblem determines how
this point is further corrected using objective-gradient information. A
direct assignment yields the basic update associated with a given geometry,
whereas nontrivial approximations, such as one step of a fixed-point
iteration, introduce additional correction or acceleration. The multiplier
recursion complements these choices by specifying the momentum state.

Accordingly, a broad class of existing optimizers can be organized as
different combinations of residual geometry and $\theta$-subproblem
approximation, as summarized in Table~\ref{tab:aim_design_space}. Beyond
providing a unified interpretation, this view also suggests a systematic
design principle for new optimizers: once a geometry is chosen through
$\psi$, alternative subproblem approximations can be used to construct new
correction or acceleration mechanisms. For example, combining matrix
spectral-norm or infinity-norm geometry with fixed-point correction suggests
accelerated Muon- or Lion-type variants. Thus, the unexplored combinations
in Table~\ref{tab:aim_design_space} define a structured design space for
future optimizer development.

\begin{table*}[t]
\centering
\caption{Representative optimizer instances within the AIM design space.
The residual penalty $\psi$ determines the update geometry, while the
approximation of the $\theta$-subproblem determines the associated
correction or acceleration mechanism.}
\label{tab:aim_design_space}

\small
\setlength{\tabcolsep}{5pt}
\renewcommand{\arraystretch}{1.22}

\begin{tabular}{
@{}
>{\centering\arraybackslash}m{0.235\textwidth}
>{\centering\arraybackslash}m{0.245\textwidth}
>{\centering\arraybackslash}m{0.300\textwidth}
>{\centering\arraybackslash}m{0.165\textwidth}
@{}
}

\toprule

\textbf{Residual geometry}
&
\textbf{Direct assignment}
&
\textbf{One fixed-point step}
&
\textbf{Alternative approximations}
\\

\midrule

Euclidean Norm
$\psi(\cdot)=\Vert\cdot\Vert^{2}$
&
SGD-M
&
Nesterov momentum
&
Open direction
\\

\specialrule{0.35pt}{4pt}{4pt}

Weighted Norm
$\psi(\cdot)=\Vert\cdot\Vert_{Q}^{2}$
&
Adam, RAD, and other adaptive optimizers
&
NAdam, RADAR, and other accelerated adaptive optimizers
&
Open direction
\\

\specialrule{0.35pt}{4pt}{4pt}

Spectral Norm
$\psi(\cdot)=\Vert\cdot\Vert_{2}^{2}$
&
Muon
&
Open direction
&
Open direction
\\

\specialrule{0.35pt}{4pt}{4pt}

Infinity Norm
$\psi(\cdot)=\Vert\cdot\Vert_{\infty}^{2}$
&
Lion
&
Open direction
&
Open direction
\\

\bottomrule
\end{tabular}
\end{table*}

\section{Relativistic Adaptive Gradient Descent with Accelerated Residual}
\label{sec:proposed_algorithm}

Building on the AIM framework developed in Section~\ref{sec:framework}, we propose \textbf{R}elativistic \textbf{A}daptive gradient \textbf{D}escent with \textbf{A}ccelerated \textbf{R}esidual (RADAR). RADAR instantiates the two AIM design axes with a relativistic adaptive residual geometry and a decoupled residual correction mechanism. The geometry inherits the speed-limiting and coordinate-wise adaptivity of RAD~\cite{lyu2024conformal}, while the correction refines the tentative descent point through the gradient--momentum mismatch. RADAR further incorporates gradient-difference momentum filtering, which enables the momentum estimate to respond more rapidly to changes in the gradient trajectory. Together, these components yield a practical optimizer that preserves the structure-aware geometry of RAD, realizes the residual-correction principle identified by AIM, and enhances momentum responsiveness through gradient-difference filtering.

\subsection{Algorithm Design}
\label{subsec:algorithm_design}

Following the practical indexing convention of adaptive optimizers, each iteration first uses the current stochastic gradient to update the momentum estimate and the adaptive geometry, and then performs the parameter update. This implementation corresponds to an index-shifted practical version of the AIM update rules.

% We present the design of RADAR according to the two axes identified by AIM: the residual geometry of the $y$-subproblem and the residual correction of the $\theta$-subproblem. For conceptual clarity, we first describe how RADAR instantiates these two axes, and then introduce the second-order momentum filtering used to compute the momentum estimate in practice. Algorithm~\ref{alg:radar} summarizes the actual implementation order, where the current stochastic gradient is first used to update the momentum estimate and the adaptive geometry before the parameter update is performed.

\paragraph{Relativistic adaptive geometry.}
Section~\ref{subsec:adaptive_geometry} shows that the $y$-subproblem determines the update geometry. Given a mini-batch \(\mathcal{B}_k=\{\xi_{k,i}\}_{i=1}^{\mathcal{B}_k}\), let \(g_k = \frac{1}{\vert \mathcal{B}_k\vert}\sum_{\xi\in\mathcal{B}_k}\nabla l(\theta_k,\xi)\) be the stochastic gradient estimator of \(\nabla\mathcal{L}(\theta_k)\). RADAR updates the second-moment estimate by
\begin{equation}
v_{k+1}
=
\beta_2 v_k+(1-\beta_2)g_k^2 ,
\label{eq:radar_v_update}
\end{equation}
where all vector operations are element-wise. We define the relativistic adaptive geometry matrix as
\begin{equation}
R_{k+1}
:=
\mathrm{Diag}\bigl(\sqrt{\delta^2 v_{k+1}+\zeta}\bigr),
\label{eq:radar_r_matrix}
\end{equation}
where \(\delta>0\) is the speed coefficient and \(\zeta\in(0,1]\) is the symplectic factor. The speed coefficient controls the strength of gradient normalization, and a larger \(\delta\) imposes stronger coordinate-wise speed limitation. The symplectic factor controls the adaptivity level and is inherited from the relativistic Hamiltonian interpretation of RAD~\cite{lyu2024conformal}.

Given the filtered momentum estimate \(m_{k+1}\), under this geometry, the AIM $y$-subproblem gives the tentative descent point as follows:
\begin{equation}
y_{k+1}
=
\theta_k
-
\eta R_{k+1}^{-1}m_{k+1}.
\label{eq:radar_y_update}
\end{equation}

\paragraph{Decoupled residual correction.}
The AIM framework shows that the $\theta$-subproblem corrects the tentative descent point using the gradient--momentum mismatch. Under relativistic adaptive geometry \eqref{eq:radar_r_matrix}, this correction takes the preconditioned form \(R_{k+1}^{-1}(m_{k+1}-g_k)\). Instead of tying its scale to the main learning rate and the momentum coefficient as in \eqref{eq:theta_nesterov_approx}, we introduce an independent correction coefficient \(\ell\):
\begin{equation}
\begin{aligned}
\theta_{k+1}
&=
y_{k+1}
+
\ell R_{k+1}^{-1}(m_{k+1}-g_k) \\
&=
\underbrace{\theta_k-\eta R_{k+1}^{-1}m_{k+1}}_{\text{relativistic momentum descent}}
-
\underbrace{\ell R_{k+1}^{-1}(g_k-m_{k+1})}_{\text{decoupled residual correction}} .
\end{aligned}
\label{eq:radar_theta_update}
\end{equation}
When \(\ell=0\), RADAR reduces to a relativistic adaptive momentum update. When \(\ell>0\), the residual term provides an acceleration-like correction to the parameter update.

\paragraph{Gradient-difference momentum correction.}
The residual correction in \eqref{eq:radar_theta_update} depends on the discrepancy between the current stochastic gradient \(g_k\) and the multiplier-like momentum estimate \(m_{k+1}\). RADAR therefore augments the standard exponential moving average with a gradient-difference correction:
\begin{equation}
m_{k+1}
=
\beta_1 m_k
+
(1-\beta_1)g_k
+
\gamma(g_k-g_{k-1}),
\label{eq:radar_m_update}
\end{equation}
where \(\gamma\geq 0\) controls the contribution of recent gradient variation. When \(\gamma=0\), \eqref{eq:radar_m_update} reduces to the standard exponential moving-average momentum update.

Following the frequency-domain perspective that interprets momentum methods as filters acting on stochastic gradients~\citep{li2025performance}, we further characterize the fixed-coefficient recursion in \eqref{eq:radar_m_update}. It can be equivalently written as
\begin{equation}
m_{k+1}
=
\beta_1m_k
+
(1-\beta_1+\gamma)g_k
-
\gamma g_{k-1}.
\end{equation}

Although RADAR and Adan share the same momentum formula, they arise from different perspectives. Adan interprets the gradient-difference term as a form of Nesterov acceleration, whereas AIM interprets momentum as a multiplier variable and derives the corresponding filtering rule from a filtering-and-control perspective. In AIM, Nesterov acceleration instead appears as a one-step fixed-point approximation to the $\theta$-subproblem. Thus, the key distinction lies not in the momentum formula itself, but in its derivation and interpretation within the AIM framework.

Defining \(\widetilde m_k=m_{k+1}\), the corresponding transfer function is
\begin{equation}
H(z)
=
\frac{\widetilde M(z)}{G(z)}
=
\frac{1-\beta_1+\gamma-\gamma z^{-1}}
     {1-\beta_1z^{-1}}.
\end{equation}
The recursion therefore contains one pole and one zero. The exponential moving-average branch provides recursive low-pass behavior, whereas the gradient-difference branch emphasizes recent changes in the stochastic gradient.

Combining relativistic adaptive geometry, decoupled residual correction, and second-order momentum filtering gives Algorithm~\ref{alg:radar}.

\begin{algorithm}[t]
\caption{Relativistic Adaptive Gradient Descent with Accelerated Residual (RADAR)}
\label{alg:radar}
\small
\begin{algorithmic}[1]
\REQUIRE Initial parameter \(\theta_0\), learning rate \(\eta>0\), correction coefficient \(\ell\geq0\), momentum coefficients \(\beta_1,\beta_2\in[0,1)\), filtering coefficient \(\gamma\geq0\), speed coefficient \(\delta>0\), symplectic factor \(\zeta\in(0,1]\), maximum iteration number \(T\)
\STATE Initialize \(m_0=0\), \(v_0=0\), and \(g_{-1}=0\)
\FOR{\(k=0,1,\ldots,T-1\)}
    \STATE Draw a mini-batch \(\mathcal{B}_k\) and compute \(g_k=\frac{1}{\vert \mathcal{B}_k\vert}\sum_{\xi\in\mathcal{B}_k}\nabla l(\theta_k,\xi)\)
    \STATE \(m_{k+1}=\beta_1m_k+(1-\beta_1)g_k+\gamma(g_k-g_{k-1})\)
    \STATE \(v_{k+1}=\beta_2v_k+(1-\beta_2)g_k^2\)
    \STATE \(R_{k+1}=\mathrm{Diag}\bigl(\sqrt{\delta^2v_{k+1}+\zeta}\bigr)\)
    \STATE \(\theta_{k+1}=\theta_k-\eta R_{k+1}^{-1}m_{k+1}
    -\ell R_{k+1}^{-1}(g_k - m_{k+1})\)
\ENDFOR
\RETURN \(\theta_T\)
\end{algorithmic}
\end{algorithm}

\subsection{Convergence Analysis}
\label{subsec:convergence}

We present the convergence analysis of RADAR in Algorithm~\ref{alg:radar}. The proof uses a variance-perturbed Lyapunov drift argument that jointly controls the objective value, the momentum residual, and the successive parameter displacement.

\begin{assumption}[Smoothness and lower boundedness]
\label{ass:smoothness}
For any data sample $\xi$, the sample-wise loss $l(\cdot,\xi)$ is $L$-smooth. Moreover, $\mathcal{L}$ is lower bounded by $\mathcal{L}_*$.
\end{assumption}

\begin{assumption}[Unbiased stochastic gradient and bounded variance]
\label{ass:unbiased_variance}
Let $g_k$ be the mini-batch stochastic gradient, where the samples are drawn independently. We assume that the single-sample stochastic gradient is unbiased and has bounded variance:
\[
\mathbb{E}[\nabla l(\theta,\xi)]
=
\nabla\mathcal{L}(\theta),
\qquad
\mathbb{E}
\left\|
\nabla l(\theta,\xi)-\nabla\mathcal{L}(\theta)
\right\|^2
\leq
\sigma^2,
\quad
\forall \theta .
\]
\end{assumption}

\begin{assumption}[Bounded adaptive geometry]
\label{ass:bounded_geometry}
Let $R_k$ be the relativistic adaptive geometry matrix in RADAR as defined in \eqref{eq:radar_r_matrix}. There exist constants $0<\nu_{\min}\leq\nu_{\max}<\infty$ such that
\begin{equation}
\nu_{\min} I
\preceq
R_k^{-1}
\preceq
\nu_{\max} I,
\quad
\forall k.
\label{eq:bounded_geometry}
\end{equation}
\end{assumption}

We define $\varepsilon_k:=\nabla\mathcal{L}(\theta_k)-g_k$, $r_{k+1}:=\nabla\mathcal{L}(\theta_k)-m_{k+1}$ and consider the Lyapunov function as follows:
\begin{equation}
V_k
:=
\mathcal{L}(\theta_k)
+
c_1\|r_k\|^2
+
c_2\|\theta_k-\theta_{k-1}\|^2 ,
\label{eq:radar_lyapunov_function}
\end{equation}
where $c_1,c_2>0$ are specified in Appendix~\ref{app:radar_convergence}.

\begin{lemma}[Lyapunov drift bound]
\label{lem:radar_lyapunov_descent}
Suppose Assumptions~\ref{ass:smoothness}--\ref{ass:bounded_geometry} hold. If the parameters $\eta,\ell,\beta_1,\gamma$ are chosen in the stable regime specified in Appendix~\ref{app:radar_convergence}, then there exist constants $d_1,d_2,C_v>0$ such that
\begin{equation}
\mathbb{E}V_{k+1}-\mathbb{E}V_k
\leq
-d_1\mathbb{E}\|r_{k+1}\|^2
-
d_2\mathbb{E}\|\nabla\mathcal{L}(\theta_k)\|^2
+
C_v\sigma^2
\left(
\frac{1}{\vert \mathcal{B}_k\vert}
+
\frac{1}{\vert \mathcal{B}_{k-1}\vert}
\right),
\label{eq:radar_lyapunov_descent}
\end{equation}
where $\mathcal{B}_{-1}:=\mathcal{B}_0$,\(d_1,d_2>0\) are the descent coefficients for the momentum residual and the stationarity measure,
while \(C_v>0\) quantifies the amplification of stochastic-gradient variance.
\end{lemma}

Lemma~\ref{lem:radar_lyapunov_descent} is a variance-perturbed Lyapunov drift bound. It reduces to a standard Lyapunov descent condition in the full-gradient case and yields a variance-controlled stationary neighborhood in the stochastic case. We define $R(T) = \frac{1}{T}
\sum_{k=0}^{T-1}
\mathbb{E}\|\nabla\mathcal{L}(\theta_k)\|^2$, and give its bound as follows.

\begin{theorem}[Sublinear convergence of RADAR]
\label{thm:radar_convergence}
Under the conditions of Lemma~\ref{lem:radar_lyapunov_descent}, the sequence generated by Algorithm~\ref{alg:radar} satisfies
\begin{equation}
R(T)
\leq
\frac{\mathbb{E}V_0-\mathcal{L}_*}{d_2T}
+
\frac{C_v\sigma^2}{d_2T}
\sum_{k=0}^{T-1}
\left(
\frac{1}{\vert \mathcal{B}_k\vert}
+
\frac{1}{\vert \mathcal{B}_{k-1}\vert}
\right).
\label{eq:radar_sublinear_rate}
\end{equation}
\end{theorem}

Therefore, with a fixed mini-batch size, RADAR converges to a variance-controlled stationary neighborhood. If full gradients are used or if the batch-size schedule makes the accumulated variance term uniformly bounded, RADAR achieves an \(O(1/T)\) average stationarity bound. The proofs of Lemma~\ref{lem:radar_lyapunov_descent} and Theorem~\ref{thm:radar_convergence} are provided in Appendix~\ref{app:radar_convergence}.

\section{Numerical Experiments}
\label{sec:experiments}

We evaluate RADAR on supervised vision learning, language modeling, and reinforcement learning to assess whether the AIM-derived optimizer RADAR provides consistent gains across training regimes with different model architectures, data modalities, and optimization dynamics. We compare RADAR with RAD, Adam, AdamW, NAdam,Adan\cite{xie2024adan},AdaBelief\cite{zhuang2020adabelief} and Lion\cite{chen2023symbolic}. For fair comparison, We conducted learning-rate searches with the same task-specific tuning budget for all optimizers on each of the four supervised learning tasks. After selecting the optimal learning rate for each optimizer, we retrained all models from scratch using five independent random seeds, except for the reinforcement learning tasks, for which ten seeds were used. The detailed experimental settings are reported in Appendix~\ref{app:extra_experiments}. Except for the Qwen experiment using 4$\times$A100,all reported experiments are conducted on a single NVIDIA RTX 5090 GPU.

\subsection{Experimental Setup}
\label{subsec:experimental_setup}

\paragraph{Supervised vision learning.}
We evaluate the image classification performance on CIFAR-10 and CIFAR-100~\cite{krizhevsky2009learning}. We train a Vision Transformer (ViT, 6.3M parameters)~\cite{dosovitskiy2021image} on CIFAR-10 and a ResNet-50~\cite{he2016deep} on CIFAR-100 with standard data augmentation, following the training configuration of ROOT~\cite{he2025root}. Each optimizer is trained for 100 epochs. We report the mean and standard deviation of test loss.

\paragraph{Language modeling.}
We evaluate GPT-2 (124M)~\cite{radford2019language} in both pre-training and fine-tuning settings. For pre-training, we train GPT-2 from scratch on WikiText-103~\cite{merity2017pointer} for 10,000 update steps with a block size of 512 and an effective batch size of 64, corresponding to approximately 328M training tokens. For fine-tuning, we fine-tune GPT-2 on WikiText-2~\cite{merity2017pointer} for 3 epochs. To assess the scalability of the compared optimizers, we additionally train Qwen2.5-0.5B~\cite{yang2024qwen25} from scratch on a 10B-token subset of FineWeb-Edu~\cite{penedo2024fineweb} with a total training budget of 1B tokens. We report test perplexity (PPL), where lower is better, together with the mean and standard deviation over multiple random seeds. Detailed language-modeling settings are provided in Appendix~\ref{sec:language}.

\paragraph{Reinforcement learning.}
We evaluate long-horizon continuous control on MuJoCo benchmarks~\cite{todorov2012mujoco,brockman2016openai}. In the main text, we report Soft Actor-Critic (SAC)~\cite{haarnoja2018soft} results on HalfCheetah-v4 and Walker2d-v4. All optimizers are plugged into the same GOPS~\cite{wang2023gops} implementation with identical training and evaluation protocols. We report the total average return (TAR) and standard deviation over 10 random seeds.

 % while Proximal Policy Optimization (PPO)~\cite{schulman2017proximal} results on Ant-v4 and Humanoid-v4 are provided in the Appendix~\ref{sec:reinforcement}
% \begin{figure*}[t]
%     \centering
%     \includegraphics[width=\linewidth]{figures.pdf}
%     \caption{Test-loss curves of different optimizers on supervised vision learning tasks.}
%     \label{fig:optimizer_compare}
% \end{figure*}

\subsection{Main Results}
\label{subsec:main_results}
Table~\ref{tab:main_results} summarizes the main results across supervised vision learning, language modeling, and reinforcement learning. RADAR achieves the best mean performance in all seven evaluated settings, spanning different model architectures, data modalities, and optimization objectives, demonstrating consistent effectiveness across diverse learning problems.

For supervised vision learning, RADAR achieves the lowest test loss on both CIFAR-10 with ViT and CIFAR-100 with ResNet-50. On CIFAR-10, RADAR obtains $0.8011\pm0.0112$, outperforming the second-best Adan at $0.8029\pm0.0086$. On CIFAR-100, RADAR achieves $1.6190\pm0.0088$, compared with $1.6244\pm0.0100$ for the second-best Adam. Although the margins are modest, RADAR consistently ranks first across both convolutional and Transformer-based architectures.

The advantage of RADAR is more pronounced in language-model pre-training. On WikiText-103 with GPT-2, RADAR achieves the best test perplexity of $22.4963\pm0.0933$, outperforming the second-best Adan at $23.0514\pm0.1909$. As shown in Figure~\ref{fig:train_loss}, RADAR also maintains the lowest training loss in the later stage of pre-training, indicating a sustained optimization advantage.

This trend persists at a larger scale on FineWeb-Edu-10B with Qwen2.5-0.5B, where RADAR achieves the best test perplexity of $30.3029\pm0.2812$, followed by RAD at $30.8392\pm0.7031$. Figure~\ref{fig:train_loss} likewise shows that RADAR reaches the lowest late-stage training loss. The consistent gains across GPT-2-124M and Qwen2.5-0.5B suggest that RADAR scales favorably to larger models and training budgets.

\begin{figure*}[t] 
\centering 
\includegraphics[width=\linewidth]{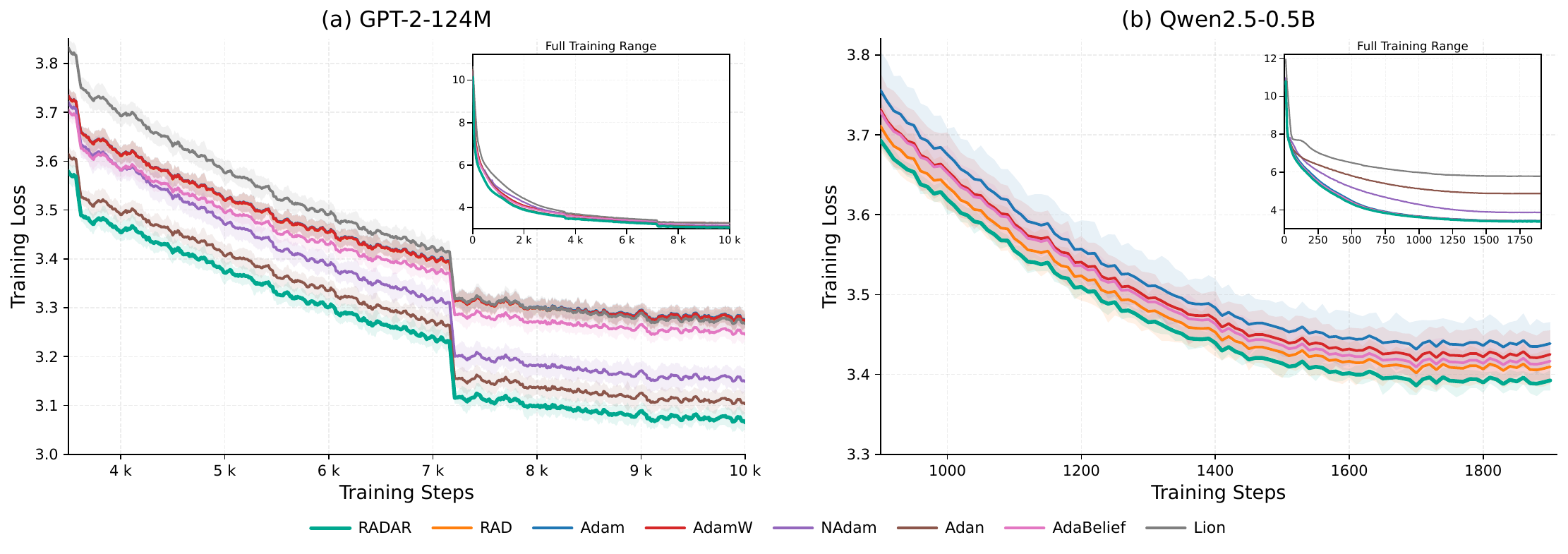} 
\caption{Training-loss curves of different optimizers on language-model pre-training tasks.}
\label{fig:train_loss} 
\end{figure*} 

For GPT-2 fine-tuning on WikiText-2, RADAR again achieves the best mean perplexity of $21.5358\pm0.0082$, outperforming the second-best NAdam at $21.5513\pm0.0054$, despite smaller differences among the strongest methods. Together with the pre-training results, this demonstrates the effectiveness of RADAR in both training-from-scratch and pretrained-model adaptation regimes. Additional language-modeling results are provided in Appendix~\ref{sec:language}.

RADAR also performs consistently well in reinforcement learning. On HalfCheetah-v4 with SAC, RADAR achieves the highest total average return of $9569\pm926$, exceeding the second-best AdaBelief at $9194\pm828$. On Walker2d-v4, RADAR obtains $3580\pm282$, compared with $3362\pm776$ for the second-best NAdam, while exhibiting substantially lower variability than other competitive methods.

% Additional reinforcement-learning results are provided in Appendix~\ref{sec:reinforcement}.

Overall, RADAR ranks first across all evaluated vision, language-modeling, and reinforcement-learning settings, with particularly notable gains in pre-training and reinforcement learning and consistent improvements in vision and language-model fine-tuning.

% ==================== Table ====================
\begin{table*}[t]
\centering

\caption{Main results across supervised vision learning, language modeling,
and reinforcement learning. Results are reported as mean $\pm$ standard
deviation. Best results are in bold and second-best results are underlined.}
\label{tab:main_results}

\small

% Horizontal spacing
\setlength{\tabcolsep}{2pt}

% Do not use large arraystretch here; otherwise the mean/std gap becomes large
\renewcommand{\arraystretch}{1.02}

\begin{tabular*}{\textwidth}{
@{\extracolsep{\fill}}
>{\raggedright\arraybackslash}m{0.170\textwidth}
C{0.060\textwidth}
*{8}{C{0.079\textwidth}}
@{}
}

\toprule

\textbf{Task}
&
\textbf{Metric}
&
\textbf{RADAR}
&
\textbf{RAD}
&
\textbf{Adam}
&
\textbf{AdamW}
&
\textbf{NAdam}
&
\textbf{Adan}
&
\textbf{AdaBelief}
&
\textbf{Lion}
\\

\midrule

% ============================================================
% CIFAR-10
% ============================================================

\multirow[c]{2}{*}{
    \makecell[l]{CIFAR-10,\\ ViT(6.3M)}
}
&
\multirow[c]{2}{*}{%
    \makecell[c]{Test\\Loss $\downarrow$}
}
&
\best{0.8011}
&
0.8307
&
0.8240
&
0.8339
&
0.8200
&
\second{0.8029}
&
0.8322
&
0.8063
\\

&
&
$\pm$\best{0.0112}
&
$\pm 0.0051$
&
$\pm 0.0078$
&
$\pm 0.0039$
&
$\pm 0.0107$
&
\second{$\pm 0.0086$}
&
$\pm 0.0078$
&
$\pm 0.0089$
\\

\addlinespace[4pt]

% ============================================================
% CIFAR-100
% ============================================================

\multirow[c]{2}{*}{
    \makecell[l]{CIFAR-100,\\ResNet-50}
}
&
\multirow[c]{2}{*}{%
    \makecell[c]{Test\\Loss $\downarrow$}
}
&
\best{1.6190}
&
1.6787
&
\second{1.6244}
&
1.6771
&
1.6326
&
1.6338
&
1.6831
&
1.7078
\\

&
&
$\pm$\best{0.0088}
&
$\pm 0.0108$
&
\second{$\pm 0.0100$}
&
$\pm 0.0191$
&
$\pm 0.0122$
&
$\pm 0.0068$
&
$\pm 0.0086$
&
$\pm 0.0152$
\\

\midrule

% ============================================================
% WikiText-103
% ============================================================

\multirow[c]{2}{*}{
    \makecell[l]{WikiText-103,\\GPT-2 Small}
}
&
\multirow[c]{2}{*}{%
    \makecell[c]{Test\\PPL $\downarrow$}
}
&
\best{22.4963}
&
26.0321
&
26.0733
&
26.0281
&
23.9754
&
\second{23.0514}
&
25.5079
&
26.1346
\\

&
&
$\pm$\best{0.0933}
&
$\pm 0.0331$
&
$\pm 0.0321$
&
$\pm 0.0433$
&
$\pm 0.2368$
&
\second{$\pm 0.1909$}
&
$\pm 0.0738$
&
$\pm 0.2068$
\\

\addlinespace[4pt]

% ============================================================
% FineWeb-Edu / Qwen
% ============================================================

\multirow[c]{2}{*}{
    \makecell[l]{FineWeb-Edu-10B,\\Qwen2.5-0.5B}
}
&
\multirow[c]{2}{*}{PPL $\downarrow$}
&
\best{30.3029}
&
\second{30.8392}
&
31.7604
&
31.3155
&
49.1373
&
133.2456
&
31.0569
&
331.9066
\\

&
&
$\pm$\best{0.2812}
&
\second{$\pm 0.7031$}
&
$\pm 0.9650$
&
$\pm 0.8857$
&
$\pm 2.1917$
&
$\pm 15.2367$
&
$\pm 0.6742$
&
$\pm 50.1679$
\\

\addlinespace[4pt]

% ============================================================
% WikiText-2
% ============================================================

\multirow[c]{2}{*}{
    \makecell[l]{WikiText-2 SFT, \\GPT-2 Small}
}
&
\multirow[c]{2}{*}{%
    \makecell[c]{Test\\PPL $\downarrow$}
}
&
\best{21.5358}
&
21.5744
&
21.5768
&
21.5765
&
\second{21.5513}
&
21.6465
&
21.5578
&
21.7824
\\

&
&
$\pm$\best{0.0082}
&
$\pm 0.0133$
&
$\pm 0.0132$
&
$\pm 0.0129$
&
\second{$\pm 0.0054$}
&
$\pm 0.0092$
&
$\pm 0.0169$
&
$\pm 0.0412$
\\

\midrule

% ============================================================
% HalfCheetah
% ============================================================

\multirow[c]{2}{*}{
    \makecell[l]{HalfCheetah-v4,\\SAC}
}
&
\multirow[c]{2}{*}{TAR $\uparrow$}
&
\best{9569}
&
8996
&
8557
&
8828
&
8582
&
8127
&
\second{9194}
&
8251
\\

&
&
$\pm $\best{926}
&
$\pm 1025$
&
$\pm 939$
&
$\pm 790$
&
$\pm 1493$
&
$\pm 559$
&
\second{$\pm 828$}
&
$\pm 828$
\\

\addlinespace[4pt]

% ============================================================
% Walker2d
% ============================================================

\multirow[c]{2}{*}{
    \makecell[l]{Walker2d-v4,\\SAC}
}
&
\multirow[c]{2}{*}{TAR $\uparrow$}
&
\best{3580}
&
2687
&
2798
&
3196
&
\second{3362}
&
2561
&
2729
&
2877
\\

&
&
$\pm $\best{282}
&
$\pm 932$
&
$\pm 1117$
&
$\pm 1124$
&
\second{$\pm 776$}
&
$\pm 1318$
&
$\pm 980$
&
$\pm 1025$
\\

\bottomrule

\end{tabular*}
\end{table*}

\section{Conclusion}

In this paper, we propose the AIM framework, which interprets momentum as a residual-driven multiplier-like correction and separates update geometry from acceleration form. Based on this view, we develop RADAR, a new optimizer that combines relativistic adaptive geometry, decoupled residual correction, and second-order momentum filtering. We establish stochastic convergence via a variance-perturbed Lyapunov analysis, and experiments on vision, language modeling, and reinforcement learning tasks demonstrate the effectiveness of RADAR. This study is currently limited to representative momentum-based optimizers and standard stochastic convergence settings. Future work will continue to explore optimizer design within the AIM framework. On the theoretical side, we aim to establish convergence guarantees under weaker conditions by relaxing the current analytical assumptions. On the empirical side, we plan to further evaluate RADAR in larger-scale models, longer training regimes, and more diverse learning scenarios to better understand its scalability and general applicability.

\clearpage

%%%%%%%%%%%%%%%%%%%%%%%%%%%%%%%%%%%%%%%%%%%%%%%%%%%%%%%%%%%%

\bibliographystyle{unsrt}   % 设置参考文献样式
\bibliography{ref}    % 引用 .bib 文件（省略后缀）

\newpage
\appendix
\section{Technical Proofs and Derivations}
\label{app:technical_proofs}

\subsection{Derivation of Multiplier Update Rules}
\label{app:multiplier_update}

We first explain the multiplier update rule in the classical ADMM update \eqref{eq:admm_multiplier_update}. The optimality condition of the $z$-subproblem \eqref{eq:admm_z_update} gives
\begin{equation}
    0
    \in
    \partial \varphi(z_{k+1})
    +
    B^\top \lambda_k
    +
    \rho B^\top(Ax_{k+1}+Bz_{k+1}-c).
\end{equation}
When $\tau=1$, the multiplier update \eqref{eq:admm_multiplier_update} becomes
\begin{equation}
    \lambda_{k+1}
    =
    \lambda_k+\rho(Ax_{k+1}+Bz_{k+1}-c).
    \label{eq:new_mul}
\end{equation}
Combining the above two relations yields
\begin{equation}
    0
    \in
    \partial \varphi(z_{k+1})
    +
    B^\top\lambda_{k+1}.
\end{equation}
Therefore, the updated pair $(z_{k+1},\lambda_{k+1})$ satisfies the stationarity condition of Problem~\eqref{eq:admm_problem} with respect to $z$. This shows that the ADMM multiplier update can be viewed as a residual-driven dual correction toward the KKT condition.

We next derive the multiplier update rule in AIM. The optimality condition of the $\theta$-subproblem gives
\begin{equation}
    0
    \in
    \nabla \mathcal{L}(\theta_{k+1})
    -
    m_k
    -
    \frac{\rho_2}{2}\partial\psi(y_{k+1}-\theta_{k+1}).
\end{equation}
Equivalently, there exists
$d_{k+1}\in\partial\psi(y_{k+1}-\theta_{k+1})$
such that
\begin{equation}
    \nabla \mathcal{L}(\theta_{k+1})
    -
    m_k
    -
    \frac{\rho_2}{2}d_{k+1}
    =
    0.
    \label{eq:aim_theta_stationarity_app}
\end{equation}
For the splitting problem \eqref{eq:splitting_problem}, the stationarity condition with respect to $\theta$ is
\begin{equation}
    \nabla \mathcal{L}(\theta)-m=0 .
\end{equation}
Thus, a full residual-driven multiplier correction would update
\begin{equation}
    m_{k+1}
    =
    m_k+\frac{\rho_2}{2}d_{k+1},
\end{equation}
which, together with \eqref{eq:aim_theta_stationarity_app}, gives $m_{k+1}=\nabla\mathcal{L}(\theta_{k+1})$ and hence satisfies the above stationarity condition at $\theta_{k+1}$.

In AIM, we use a relaxed multiplier update with stepsize $1-\beta_1$:
\begin{equation}
    m_{k+1}
    =
    m_k
    +
    (1-\beta_1)\frac{\rho_2}{2}d_{k+1}.
    \label{eq:aim_multiplier_update_app}
\end{equation}
Substituting \eqref{eq:aim_theta_stationarity_app} into
\eqref{eq:aim_multiplier_update_app} yields
\begin{equation}
    m_{k+1}
    =
    m_k
    +
    (1-\beta_1)
    \bigl(\nabla\mathcal{L}(\theta_{k+1})-m_k\bigr)
    =
    \beta_1 m_k
    +
    (1-\beta_1)\nabla\mathcal{L}(\theta_{k+1}).
\end{equation}
Therefore, the AIM multiplier update recovers the standard exponential moving average form of momentum from a relaxed residual-driven multiplier correction.

\subsection{Proof of Theorem~\ref{thm:nesterov}}
\label{app:nesterov_details}

\begin{proof}
We prove the result by eliminating the auxiliary descent variable \(y\). Under the Euclidean residual penalty \(\psi(r)=\|r\|^2\), the \(y\)-subproblem in
Algorithm~\ref{alg:admm_inspired_momentum} becomes
\[
y_{k+1}
=
\arg\min_y
\left\{
\langle m_k,y-\theta_k\rangle
+
\frac{\rho_1}{2}\|y-\theta_k\|^2
\right\}.
\]
Its first-order optimality condition is
\[
m_k+\rho_1(y_{k+1}-\theta_k)=0,
\]
and hence
\begin{equation}
y_{k+1}
=
\theta_k-\frac{1}{\rho_1}m_k .
\label{eq:app_y_update}
\end{equation}
Taking \(\rho_1=1/\eta\), we obtain
\begin{equation}
y_{k+1}
=
\theta_k-\eta m_k .
\label{eq:app_y_update_eta}
\end{equation}
Thus, \(y_{k+1}\) is the tentative momentum descent point.

We first consider the Nesterov-type approximation of the \(\theta\)-subproblem used in \eqref{eq:theta_nesterov_approx}:
\begin{equation}
\theta_{k+1}
=
y_{k+1}
-
\eta(1-\beta_1)\bigl(\nabla \mathcal{L}(\theta_k)-m_k\bigr).
\label{eq:app_theta_nesterov_approx}
\end{equation}
Substituting \eqref{eq:app_y_update_eta} into \eqref{eq:app_theta_nesterov_approx} gives
\begin{equation}
\theta_{k+1}
=
\theta_k
-
\eta m_k
-
\eta(1-\beta_1)\bigl(\nabla \mathcal{L}(\theta_k)-m_k\bigr).
\label{eq:app_nesterov_decomposition}
\end{equation}
Equivalently, by defining the corrected momentum direction as follow:
\[
\widetilde m_k
:=
m_k+(1-\beta_1)\bigl(\nabla \mathcal{L}(\theta_k)-m_k\bigr),
\]
we can write
\begin{equation}
\theta_{k+1}
=
\theta_k-\eta \widetilde m_k .
\label{eq:app_nesterov_direction}
\end{equation}
Therefore, compared with the pure momentum descent direction \(m_k\), the update direction is corrected by the gradient--momentum mismatch \(\nabla \mathcal{L}(\theta_k)-m_k\). This is the Nesterov-type correction induced by the one-step approximation of the \(\theta\)-subproblem in \eqref{eq:theta_nesterov_approx}. Hence, the resulting update recovers a Nesterov-type momentum method up to the change of variables
\[
\widetilde m_k
=
m_k+(1-\beta_1)\bigl(\nabla \mathcal{L}(\theta_k)-m_k\bigr).
\]

We next consider the simpler approximation as follows:
\[
\theta_{k+1}=y_{k+1}.
\]
Combining it with \eqref{eq:app_y_update_eta} gives
\begin{equation}
\theta_{k+1}
=
\theta_k-\eta m_k .
\label{eq:app_heavy_ball_update}
\end{equation}
Together with the momentum recursion derived from the multiplier update,
\[
m_{k+1}
=
\beta_1m_k+(1-\beta_1)\nabla \mathcal{L}(\theta_{k+1}),
\]
or its usual explicit implementation using the available gradient \(\nabla \mathcal{L}(\theta_k)\), this is the
heavy-ball momentum update up to the standard indexing convention. Hence, the direct approximation
\(\theta_{k+1}=y_{k+1}\) removes the residual correction term and reduces the update to heavy-ball
momentum.
\end{proof}

\subsection{Details for Adaptive and Matrix-Norm Extensions}
\label{app:adaptive_matrix_extensions}

We first justify the adaptive update in \eqref{eq:adam_y_update}. With the weighted residual penalty as follows:
\[
\psi(r)=\|r\|_{Q_k}^2,
\qquad
Q_k=\mathrm{Diag}(\sqrt{v_k}+\epsilon),
\]
the $y$-subproblem becomes
\[
y_{k+1}
=
\arg\min_y
\left\{
\langle m_k,y-\theta_k\rangle
+
\frac{\rho_1}{2}\|y-\theta_k\|_{Q_k}^2
\right\}.
\]
Its first-order optimality condition is
\[
m_k+\rho_1Q_k(y_{k+1}-\theta_k)=0,
\]
which gives
\[
y_{k+1}
=
\theta_k-\frac{1}{\rho_1}Q_k^{-1}m_k.
\]
Taking $\rho_1=1/\eta$ yields
\[
y_{k+1}
=
\theta_k-\eta Q_k^{-1}m_k.
\]
This proves the adaptive preconditioned update used in \eqref{eq:adam_y_update}. The direct approximation $\theta_{k+1}=y_{k+1}$ gives the Adam-type update, while the Nesterov-type approximation in \eqref{eq:nadam_theta_update} gives the NAdam-type update.

We next discuss the matrix-norm case. Let $\Theta,Y,M$ denote matrix-valued counterparts of $\theta,y,m$. Under a spectral-norm geometry, the $y$-step can be interpreted as a normalized steepest-descent step for the local linear model, i.e.,
\[
\min_{\Delta}\ \langle M_k,\Delta\rangle
\qquad
\mathrm{s.t.}\quad \|\Delta\|_2\leq \eta.
\]
Using the duality between the spectral norm and the nuclear norm, an optimal normalized descent direction is given by the negative polar factor of $M_k$. If
\[
M_k=U\Sigma V^\top
\]
is the singular value decomposition of $M_k$, define
\[
\mathrm{Orthogonal}(M_k):=UV^\top .
\]
Then the matrix-valued descent step is
\[
Y_{k+1}
=
\Theta_k-\eta\,\mathrm{Orthogonal}(M_k).
\]
With the direct approximation $\Theta_{k+1}=Y_{k+1}$, we obtain
\[
\Theta_{k+1}
=
\Theta_k-\eta\,\mathrm{Orthogonal}(M_k),
\]
which is the Muon-type matrix update.

\subsection{Proofs for the Convergence Analysis of RADAR}
\label{app:radar_convergence}

\subsubsection{Technical estimates}

Recall that
\[
\varepsilon_k
=
\nabla\mathcal{L}(\theta_k)-g_k,
\qquad
r_{k+1}
=
\nabla\mathcal{L}(\theta_k)-m_{k+1}.
\]
We first record a standard mini-batch variance bound. Under Assumption~\ref{ass:unbiased_variance}, conditioned on $\theta_k$, we have
\[
\mathbb{E}[g_k\mid \theta_k]
=
\mathbb{E}\left[
\frac{1}{\vert \mathcal{B}_k\vert}\sum_{b=1}^{B_k}\nabla l(\theta_k,\xi_{k,b})
\mid \theta_k
\right]
=
\nabla\mathcal{L}(\theta_k),
\]
and
\[
g_k-\nabla\mathcal{L}(\theta_k)
=
\frac{1}{\vert \mathcal{B}_k\vert}
\sum_{b=1}^{B_k}
\left(
\nabla l(\theta_k,\xi_{k,b})
-
\nabla\mathcal{L}(\theta_k)
\right).
\]
Since the mini-batch samples are independent, the cross terms vanish. Therefore,
\begin{equation}
\mathbb{E}\|\varepsilon_k\|^2
=
\mathbb{E}\|g_k-\nabla\mathcal{L}(\theta_k)\|^2 \\
=
\frac{1}{\vert \mathcal{B}_k\vert^2}
\sum_{b=1}^{B_k}
\mathbb{E}
\left\|
\nabla l(\theta_k,\xi_{k,b})
-
\nabla\mathcal{L}(\theta_k)
\right\|^2
\leq
\frac{\sigma^2}{\vert \mathcal{B}_k\vert}.
\label{eq:app_minibatch_variance}
\end{equation}

We next derive two estimates used in the Lyapunov analysis. From the second-order momentum filtering update, i.e.,
\[
m_{k+1}
=
\beta_1m_k+(1-\beta_1)g_k+\gamma(g_k-g_{k-1}),
\]
we have
\begin{align}
r_{k+1}
&=
\nabla\mathcal{L}(\theta_k)
-
\beta_1m_k
-
(1-\beta_1+\gamma)g_k
+
\gamma g_{k-1} \notag\\
&=
\varepsilon_k-\beta_1\varepsilon_{k-1}
+
\beta_1r_k
+
(\beta_1-\gamma)(g_k-g_{k-1}).
\label{eq:app_residual_decomposition}
\end{align}
Define
\[
a_k:=\varepsilon_k-\beta_1\varepsilon_{k-1},
\qquad
b_k:=\beta_1 r_k,
\qquad
c_k:=(\beta_1-\gamma)(g_k-g_{k-1}).
\]
Then \(r_{k+1}=a_k+b_k+c_k\). By Young's inequality, i.e.,
\[
\|x+y\|^2
\leq
(1+\tau)\|x\|^2+\left(1+\frac{1}{\tau}\right)\|y\|^2 ,
\qquad \tau>0 .
\]
Taking \(x=b_k\), \(y=a_k+c_k\), and using \(\|a_k+c_k\|^2\leq 2\|a_k\|^2+2\|c_k\|^2\), it follows that
\begin{equation}
\mathbb{E}\|r_{k+1}\|^2
\leq
2(1+\frac{1}{\tau})\mathbb{E}\|\varepsilon_k-\beta_1\varepsilon_{k-1}\|^2
+
(1+\tau)\beta_1^2\mathbb{E}\|r_k\|^2
+
2(1+\frac{1}{\tau})(\beta_1-\gamma)^2
\mathbb{E}\|g_k-g_{k-1}\|^2 .
\label{eq:app_residual_young}
\end{equation}

For the first term on the right-hand side of \eqref{eq:app_residual_young}, we have
\begin{equation}
\mathbb{E}\|\varepsilon_k-\beta_1\varepsilon_{k-1}\|^2
=
\mathbb{E}\|\varepsilon_k\|^2
+
\beta_1^2\mathbb{E}\|\varepsilon_{k-1}\|^2
\leq
\frac{\sigma^2}{\vert \mathcal{B}_k\vert}
+
\frac{\sigma^2}{\vert \mathcal{B}_{k-1}\vert},
\label{eq:app_noise_pair_bound}
\end{equation}
where the first equality holds because $\mathbb{E}\langle \varepsilon_k,\varepsilon_{k-1} \rangle = 0$, and the last inequality follows from \eqref{eq:app_minibatch_variance} and \(\beta_1\in[0,1)\).

For the third term on the right-hand side of \eqref{eq:app_residual_young}, since
\[
g_k=\nabla\mathcal{L}(\theta_k)-\varepsilon_k,
\]
we have
\[
g_k-g_{k-1}
=
\nabla\mathcal{L}(\theta_k)-\nabla\mathcal{L}(\theta_{k-1})
-\varepsilon_k+\varepsilon_{k-1}.
\]
Using \(\|a+b+c\|^2\leq3\|a\|^2+3\|b\|^2+3\|c\|^2\), the smoothness of \(\mathcal{L}\), and \eqref{eq:app_minibatch_variance}, we obtain
\begin{align}
\mathbb{E}\|g_k-g_{k-1}\|^2
&\leq
3\mathbb{E}
\|\nabla\mathcal{L}(\theta_k)-\nabla\mathcal{L}(\theta_{k-1})\|^2
+
3\mathbb{E}\|\varepsilon_k\|^2
+
3\mathbb{E}\|\varepsilon_{k-1}\|^2 \notag\\
&\leq
3 L^2\mathbb{E}\|\theta_k-\theta_{k-1}\|^2
+
3\frac{\sigma^2}{\vert \mathcal{B}_k\vert}
+
3\frac{\sigma^2}{\vert \mathcal{B}_{k-1}\vert} .
\label{eq:app_gradient_difference_bound}
\end{align}

Combining \eqref{eq:app_residual_young}, \eqref{eq:app_noise_pair_bound}, and
\eqref{eq:app_gradient_difference_bound}, we obtain
\begin{equation}
\begin{aligned}
\mathbb{E}\|r_{k+1}\|^2
\leq&
(1+\tau)\beta_1^2\mathbb{E}\|r_k\|^2
+
6(1+\frac{1}{\tau})(\beta_1-\gamma)^2 L^2\mathbb{E}\|\theta_k-\theta_{k-1}\|^2\\
&+
\left[2(1+\frac{1}{\tau})+6(1+\frac{1}{\tau})(\beta_1-\gamma)^2\right]\sigma^2
\left(
\frac{1}{\vert \mathcal{B}_k\vert}+\frac{1}{\vert \mathcal{B}_{k-1}\vert}
\right),
\label{eq:app_residual_recursion}
\end{aligned}
\end{equation}

Rearranging \eqref{eq:app_residual_recursion} yields
\begin{equation}
\begin{aligned}
\mathbb{E}\|r_{k+1}\|^2-\mathbb{E}\|r_k\|^2
\leq
&
-C_r\mathbb{E}\|r_{k+1}\|^2 \\
&+
C_\theta
\mathbb{E}\|\theta_k-\theta_{k-1}\|^2
+
C_\varepsilon
\sigma^2
\left(
\frac{1}{\vert \mathcal{B}_k\vert}+\frac{1}{\vert \mathcal{B}_{k-1}\vert}
\right).
\end{aligned}
\label{eq:app_residual_estimate}
\end{equation}

where the constants are chosen as
\begin{equation}
C_r 
= 
\frac{1}{(1+\tau)\beta_1^2}-1
,
\qquad
C_\theta
=
\frac{6}{\tau\beta_1^2}(\beta_1-\gamma)^2 L^2,
\qquad
C_\varepsilon
=
\frac{2}{\tau}\left[1+3(\beta_1-\gamma)^2\right] .
\label{eq:app_residual_constants}
\end{equation}

It remains to bound the successive parameter displacement. From Algorithm~\ref{alg:radar}, we have
\begin{align}
\theta_{k+1}-\theta_k
&=
-\eta R_{k+1}^{-1}m_{k+1}
-
\ell R_{k+1}^{-1}(g_k-m_{k+1}) \notag\\
&=
(\eta-\ell)R_{k+1}^{-1}r_{k+1}
+
\ell R_{k+1}^{-1}\varepsilon_k
-
\eta R_{k+1}^{-1}\nabla\mathcal{L}(\theta_k).
\label{eq:app_step_decomposition}
\end{align}
By Assumption~\ref{ass:bounded_geometry}, we have
\[
\|x\|_{R_k^{-1}}^2\leq \nu_{\max}\|x\|^2,
\qquad
\|R_k^{-1}x\|^2\leq \nu_{\max}^2\|x\|^2 .
\]
Using Assumptions~\ref{ass:unbiased_variance} and \ref{ass:bounded_geometry}, we obtain
\begin{equation}
\mathbb{E}\|\theta_{k+1}-\theta_k\|^2
\leq
3(\eta-\ell)^2\nu_{\max}^2\mathbb{E}\|r_{k+1}\|^2
+
3\eta^2\nu_{\max}^2\mathbb{E}\|\nabla\mathcal{L}(\theta_k)\|^2
+
3\ell^2\nu_{\max}^2\frac{\sigma^2}{\vert \mathcal{B}_k\vert}.
\label{eq:app_step_estimate}
\end{equation}

\subsubsection{Proof of Lemma~\ref{lem:radar_lyapunov_descent}}

\begin{proof}
By the $ L$-smoothness of $\mathcal{L}$, we have
\begin{equation}
\mathcal{L}(\theta_{k+1})
\leq
\mathcal{L}(\theta_k)
+
\langle \nabla\mathcal{L}(\theta_k),\theta_{k+1}-\theta_k\rangle
+
\frac{ L}{2}\|\theta_{k+1}-\theta_k\|^2 .
\label{eq:app_smooth_descent}
\end{equation}
Substituting \eqref{eq:app_step_decomposition} into \eqref{eq:app_smooth_descent} gives
\begin{equation}
\begin{aligned}
\mathcal{L}(\theta_{k+1})
\leq
&
\mathcal{L}(\theta_k)
-
\eta\|\nabla\mathcal{L}(\theta_k)\|_{R_{k+1}^{-1}}^2 \\
&+
(\eta-\ell)
\langle\nabla\mathcal{L}(\theta_k),R_{k+1}^{-1}r_{k+1}\rangle
+
\ell
\langle\nabla\mathcal{L}(\theta_k),R_{k+1}^{-1}\varepsilon_k\rangle \\
&+
\frac{ L}{2}\|\theta_{k+1}-\theta_k\|^2 .
\end{aligned}
\label{eq:app_objective_descent_1}
\end{equation}
Using Young's inequality in the weighted norm and $\eta \geq l$, i.e.,
\[
\langle a,R_{k+1}^{-1}b\rangle
\leq
\frac{1}{2}\|a\|_{R_{k+1}^{-1}}^2
+
\frac{1}{2}\|b\|_{R_{k+1}^{-1}}^2 ,
\]
we have
\[
(\eta-\ell)
\langle\nabla\mathcal{L}(\theta_k),R_{k+1}^{-1}r_{k+1}\rangle
\leq
\frac{\eta-\ell}{2}
\|\nabla\mathcal{L}(\theta_k)\|_{R_{k+1}^{-1}}^2
+
\frac{\eta-\ell}{2}
\|r_{k+1}\|_{R_{k+1}^{-1}}^2 ,
\]
and
\[
\ell
\langle\nabla\mathcal{L}(\theta_k),R_{k+1}^{-1}\varepsilon_k\rangle
\leq
\frac{\ell}{2}
\|\nabla\mathcal{L}(\theta_k)\|_{R_{k+1}^{-1}}^2
+
\frac{\ell}{2}
\|\varepsilon_k\|_{R_{k+1}^{-1}}^2 .
\]
Substituting these bounds into \eqref{eq:app_objective_descent_1} yields
\begin{equation}
\begin{aligned}
\mathcal{L}(\theta_{k+1})
\leq
&
\mathcal{L}(\theta_k)
-
\frac{\eta}{2}
\|\nabla\mathcal{L}(\theta_k)\|_{R_{k+1}^{-1}}^2
+
\frac{\eta-\ell}{2}
\|r_{k+1}\|_{R_{k+1}^{-1}}^2 \\
&+
\frac{\ell}{2}
\|\varepsilon_k\|_{R_{k+1}^{-1}}^2
+
\frac{ L}{2}
\|\theta_{k+1}-\theta_k\|^2 .
\end{aligned}
\label{eq:app_objective_descent_2}
\end{equation}
By Assumption~\ref{ass:bounded_geometry} and \eqref{eq:app_minibatch_variance}, we have
\[
\|\nabla\mathcal{L}(\theta_k)\|_{R_{k+1}^{-1}}^2
\geq
\nu_{\min}
\|\nabla\mathcal{L}(\theta_k)\|^2,
\]
\[
\|r_{k+1}\|_{R_{k+1}^{-1}}^2
\leq
\nu_{\max}
\|r_{k+1}\|^2,
\]
and
\[
\mathbb{E}\|\varepsilon_k\|_{R_{k+1}^{-1}}^2
\leq
\nu_{\max}\mathbb{E}\|\varepsilon_k\|^2
\leq
\nu_{\max}\frac{\sigma^2}{\vert \mathcal{B}_k\vert}.
\]
Taking expectation in \eqref{eq:app_objective_descent_2}, we obtain
\begin{equation}
\begin{aligned}
\mathbb{E}\mathcal{L}(\theta_{k+1})
-
\mathbb{E}\mathcal{L}(\theta_k)
\leq
&
-\frac{\eta\nu_{\min}}{2}
\mathbb{E}\|\nabla\mathcal{L}(\theta_k)\|^2
+
\frac{(\eta-\ell)\nu_{\max}}{2}
\mathbb{E}\|r_{k+1}\|^2 \\
&+
\frac{ L}{2}
\mathbb{E}\|\theta_{k+1}-\theta_k\|^2
+
\frac{\ell\nu_{\max}}{2}
\frac{\sigma^2}{\vert \mathcal{B}_k\vert}.
\end{aligned}
\label{eq:app_objective_descent_final}
\end{equation}

Combining \eqref{eq:app_objective_descent_final} with the residual estimate
\eqref{eq:app_residual_estimate}, we get
\begin{equation}
\begin{aligned}
\mathbb{E}V_{k+1}-\mathbb{E}V_k
\leq
&
-\frac{\eta\nu_{\min}}{2}
\mathbb{E}\|\nabla\mathcal{L}(\theta_k)\|^2 \\
&-
\left[
c_1 C_r
-
\frac{(\eta-\ell)\nu_{\max}}{2}
\right]
\mathbb{E}\|r_{k+1}\|^2 \\
&-
(c_2-c_1C_\theta)
\mathbb{E}\|\theta_k-\theta_{k-1}\|^2
+
\left(c_2+\frac{ L}{2}\right)
\mathbb{E}\|\theta_{k+1}-\theta_k\|^2 \\
&+
\left(c_1C_\varepsilon
+
\frac{\ell\nu_{\max}}{2}\right)
\frac{\sigma^2}{\vert \mathcal{B}_k\vert}
+
c_1C_\varepsilon
\frac{\sigma^2}{\vert \mathcal{B}_{k-1}\vert},
\end{aligned}
\label{eq:app_lyapunov_before_step}
\end{equation}

Next, substituting \eqref{eq:app_step_estimate} into \eqref{eq:app_lyapunov_before_step} yields
\begin{equation}
\begin{aligned}
\mathbb{E}V_{k+1}-\mathbb{E}V_k
\leq
&
-d_1\mathbb{E}\|r_{k+1}\|^2
-
d_2\mathbb{E}\|\nabla\mathcal{L}(\theta_k)\|^2 \\
&+
(c_1C_\theta-c_2)
\mathbb{E}\|\theta_k-\theta_{k-1}\|^2
+
C_{v,+}
\frac{\sigma^2}{\vert \mathcal{B}_k\vert}
+
C_{v,-}
\frac{\sigma^2}{\vert \mathcal{B}_{k-1}\vert},
\end{aligned}
\label{eq:app_lyapunov_with_step}
\end{equation}
where
\begin{equation}
d_1
:=
c_1C_r
-
\frac{(\eta-\ell)\nu_{\max}}{2}
-
3\left(c_2+\frac{ L}{2}\right)
(\eta-\ell)^2\nu_{\max}^2,
\label{eq:app_d1_explicit}
\end{equation}
\begin{equation}
d_2
:=
\frac{\eta\nu_{\min}}{2}
-
3\left(c_2+\frac{ L}{2}\right)
\eta^2\nu_{\max}^2,
\label{eq:app_d2_explicit}
\end{equation}
and
\begin{equation}
C_{v,+}
:=
c_1C_\varepsilon
+
\frac{\ell\nu_{\max}}{2}
+
3\left(c_2+\frac{ L}{2}\right)\ell^2\nu_{\max}^2,
\qquad
C_{v,-}
:=
c_1C_\varepsilon .
\label{eq:app_Cv_explicit}
\end{equation}
Choose \(c_1,c_2>0\) and the algorithmic parameters in a stable regime such that
\[
c_1C_\theta-c_2\leq 0,
\qquad
d_1>0,
\qquad
d_2>0,
\qquad
\eta \geq l.
\]
Then the nonpositive term \((c_1C_\theta-c_2)\mathbb{E}\|\theta_k-\theta_{k-1}\|^2\) can be dropped from the upper bound. Define
\[
C_v:=\max\{C_{v,+},C_{v,-}\}>0 .
\]
We finally obtain
\begin{equation}
\mathbb{E}V_{k+1}-\mathbb{E}V_k
\leq
-d_1\mathbb{E}\|r_{k+1}\|^2
-
d_2\mathbb{E}\|\nabla\mathcal{L}(\theta_k)\|^2 
+
C_v\sigma^2
\left(
\frac{1}{\vert \mathcal{B}_k\vert}
+
\frac{1}{\vert \mathcal{B}_{k-1}\vert}
\right).
\label{eq:app_lyapunov_descent_upper}
\end{equation}
This proves Lemma~\ref{lem:radar_lyapunov_descent}.
\end{proof}

\subsubsection{Proof of Theorem~\ref{thm:radar_convergence}}

\begin{proof}
From \eqref{eq:app_lyapunov_descent_upper}, we have
\begin{equation}
d_1\mathbb{E}\|r_{k+1}\|^2
+
d_2\mathbb{E}\|\nabla\mathcal{L}(\theta_k)\|^2
\leq
\mathbb{E}V_k-\mathbb{E}V_{k+1}
+
C_v\sigma^2
\left(
\frac{1}{\vert \mathcal{B}_k\vert}
+
\frac{1}{\vert \mathcal{B}_{k-1}\vert}
\right).
\label{eq:app_one_step_for_theorem}
\end{equation}
Since \(d_1\mathbb{E}\|r_{k+1}\|^2\geq 0\), dropping it gives
\begin{equation}
d_2\mathbb{E}\|\nabla\mathcal{L}(\theta_k)\|^2
\leq
\mathbb{E}V_k-\mathbb{E}V_{k+1}
+
C_v\sigma^2
\left(
\frac{1}{\vert \mathcal{B}_k\vert}
+
\frac{1}{\vert \mathcal{B}_{k-1}\vert}
\right).
\label{eq:app_gradient_one_step}
\end{equation}
Summing \eqref{eq:app_gradient_one_step} from \(k=0\) to \(T-1\) gives
\begin{equation}
d_2
\sum_{k=0}^{T-1}
\mathbb{E}\|\nabla\mathcal{L}(\theta_k)\|^2
\leq
\mathbb{E}V_0-\mathbb{E}V_T
+
C_v\sigma^2
\sum_{k=0}^{T-1}
\left(
\frac{1}{\vert \mathcal{B}_k\vert}
+
\frac{1}{\vert \mathcal{B}_{k-1}\vert}
\right).
\label{eq:app_telescoping}
\end{equation}
Since \(V_T\geq \mathcal{L}(\theta_T)\geq \mathcal{L}_*\), we have
\[
\mathbb{E}V_0-\mathbb{E}V_T
\leq
\mathbb{E}V_0-\mathcal{L}_* .
\]
Therefore,
\begin{equation}
d_2
\sum_{k=0}^{T-1}
\mathbb{E}\|\nabla\mathcal{L}(\theta_k)\|^2
\leq
\mathbb{E}V_0-\mathcal{L}_*
+
C_v\sigma^2
\sum_{k=0}^{T-1}
\left(
\frac{1}{\vert \mathcal{B}_k\vert}
+
\frac{1}{\vert \mathcal{B}_{k-1}\vert}
\right).
\label{eq:app_telescoping_final}
\end{equation}
Dividing both sides by \(d_2T\) yields
\begin{equation}
\frac{1}{T}
\sum_{k=0}^{T-1}
\mathbb{E}\|\nabla\mathcal{L}(\theta_k)\|^2
\leq
\frac{\mathbb{E}V_0-\mathcal{L}_*}{d_2T}
+
\frac{C_v\sigma^2}{d_2T}
\sum_{k=0}^{T-1}
\left(
\frac{1}{\vert \mathcal{B}_k\vert}
+
\frac{1}{\vert \mathcal{B}_{k-1}\vert}
\right),
\end{equation}
which proves Theorem~\ref{thm:radar_convergence}.
\end{proof}

\section{Additional Details and Results on Numerical Experiments}
\label{app:extra_experiments}
\subsection{Learning-Rate Search Protocol}
\label{app:learning-rate_search_protocol}
Following the optimizer-specific recommendations, the initial learning rate of
Adan was set to $2.5$ times the task-specific base learning rate, while that of
Lion was set to $0.1$ times the base learning rate. All other optimizers used
the task-specific base learning rate directly.

For each optimizer with an initial task-specific learning rate $\eta_0$, we
evaluated the following candidate set:
\begin{equation}
\left\{
0.1\eta_0,\;
0.5\eta_0,\;
\eta_0,\;
5\eta_0,\;
10\eta_0
\right\}.
\label{eq:lr_candidates}
\end{equation}

The initial learning rates are summarized in
Table~\ref{tab:initial_learning_rates}.

\begin{table*}[t]
\centering
\caption{Initial learning rates used in the learning-rate search.}
\label{tab:initial_learning_rates}

\resizebox{\textwidth}{!}{
\begin{tabular}{lcccccccc}
\toprule
Task & RADAR & RAD & Adam & AdamW & NAdam & Adan & AdaBelief & Lion \\
\midrule

CIFAR-10, ViT (6.3M)
& $1\times10^{-3}$
& $1\times10^{-3}$
& $1\times10^{-3}$
& $1\times10^{-3}$
& $1\times10^{-3}$
& $2.5\times10^{-3}$
& $1\times10^{-3}$
& $1\times10^{-4}$
\\

CIFAR-100, ResNet-50
& $1\times10^{-3}$
& $1\times10^{-3}$
& $1\times10^{-3}$
& $1\times10^{-3}$
& $1\times10^{-3}$
& $2.5\times10^{-3}$
& $1\times10^{-3}$
& $1\times10^{-4}$
\\

Wikitext-103, GPT-2 Small
& $3\times10^{-4}$
& $3\times10^{-4}$
& $3\times10^{-4}$
& $3\times10^{-4}$
& $3\times10^{-4}$
& $7.5\times10^{-4}$
& $3\times10^{-4}$
& $3\times10^{-5}$
\\

FineWeb-Edu-10B, Qwen2.5-0.5B
& $3\times10^{-4}$
& $3\times10^{-4}$
& $3\times10^{-4}$
& $3\times10^{-4}$
& $3\times10^{-4}$
& $7.5\times10^{-4}$
& $3\times10^{-4}$
& $3\times10^{-5}$
\\

WikiText-2 SFT, GPT-2 Small
& $1\times10^{-4}$
& $1\times10^{-4}$
& $1\times10^{-4}$
& $1\times10^{-4}$
& $1\times10^{-4}$
& $2.5\times10^{-4}$
& $1\times10^{-4}$
& $1\times10^{-5}$
\\

\bottomrule
\end{tabular}
}
\end{table*}

Learning-rate tuning was performed using an independent tuning seed
($s=5$). Each optimizer was then trained again from scratch using the selected
learning rate and five independent final seeds:
\begin{equation}
\{0,1,2,3,4\}.
\end{equation}

For all tasks, the learning rate was selected using the average of the final
three validation measurements, thereby reducing sensitivity to fluctuations at
a single evaluation point. Specifically, GPT-2 pre-training evaluated each
candidate for 1,000 optimizer steps and recorded the validation loss every 100
steps, with the measurements at steps 800, 900, and 1,000 used for selection.
The ViT and ResNet-50 tasks evaluated each candidate for 10 epochs and used the
validation losses from the final three epochs. GPT-2 fine-tuning evaluated each
candidate for one complete epoch and used the final three validation perplexity
measurements collected during that epoch.

More generally, for a task-specific validation metric
$\mathcal{M}_{\mathrm{val}}$, the selection score was defined as
\begin{equation}
S(\eta)
=
\frac{1}{3}
\sum_{j=M-2}^{M}
\mathcal{M}_{\mathrm{val}}^{(j)}(\eta),
\label{eq:lr_selection_score}
\end{equation}
where $M$ is the total number of validation measurements. Validation loss was
used for GPT-2 pre-training and the two vision tasks, while validation
perplexity was used for GPT-2 fine-tuning. The candidate learning rate with the
lowest selection score was selected.

Within each task, all learning-rate trials used the same tuning seed, model
initialization, data split and data order, batch configuration, warmup setting,
weight-decay setting, and the same initial portion of the learning-rate
schedule used in the final training runs.

\subsection{Hyperparameters of All Experiments}
\label{app:Hyperparameters}

Table~\ref{tab:selected_learning_rates} summarizes the learning rates used in the reported experiments. Table~\ref{tab:Hyperparameters} summarizes the RADAR-specific hyperparameters used in the reported experiments. We select  $\beta_{1} = 0.9$ across tasks and optimizers,and $\delta = 1$, $\zeta = 1e-16$ in all the experiments for RADAR and RAD.For the second-moment coefficient $\beta_2$, the default values, either $0.999$ or $0.99$, are used for all tasks except Qwen, for which all optimizers use $0.95$.For the choice of $\gamma$ and $l$, we set $l=0.01\eta$, where $\eta$ denotes the learning rate. We use $\gamma=0.01$ for the reinforcement learning and Qwen tasks, and $\gamma=0.1$ for all other tasks.

\begin{table*}[t]
\centering
\caption{Selected learning rates across tasks and optimizers.}
\label{tab:selected_learning_rates}

\resizebox{\textwidth}{!}{
\begin{tabular}{lcccccccc}
\toprule
Task & RADAR & RAD & Adam & AdamW & NAdam & Adan & AdaBelief & Lion \\
\midrule

CIFAR-10, ViT (6.3M)
& $5\times10^{-4}$
& $5\times10^{-4}$
& $5\times10^{-4}$
& $5\times10^{-4}$
& $5\times10^{-4}$
& $1.25\times10^{-3}$
& $5\times10^{-4}$
& $1\times10^{-4}$
\\

CIFAR-100, ResNet-50
& $1\times10^{-3}$
& $5\times10^{-3}$
& $1\times10^{-3}$
& $5\times10^{-3}$
& $1\times10^{-3}$
& $2.5\times10^{-3}$
& $5\times10^{-3}$
& $5\times10^{-4}$
\\

\midrule

WikiText-103, GPT-2 Small
& $1.5\times10^{-3}$
& $3\times10^{-4}$
& $3\times10^{-4}$
& $3\times10^{-4}$
& $1.5\times10^{-3}$
& $3.75\times10^{-3}$
& $3\times10^{-4}$
& $1.5\times10^{-4}$
\\

FineWeb-Edu-10B, Qwen2.5-0.5B
& $1.5\times10^{-3}$
& $1.5\times10^{-3}$
& $1.5\times10^{-3}$
& $1.5\times10^{-3}$
& $1.5\times10^{-3}$
& $3.75\times10^{-3}$
& $3\times10^{-4}$
& $1.5\times10^{-4}$
\\

WikiText-2 SFT, GPT-2 Small
& $1\times10^{-4}$
& $1\times10^{-4}$
& $1\times10^{-4}$
& $1\times10^{-4}$
& $1\times10^{-4}$
& $2.5\times10^{-4}$
& $1\times10^{-4}$
& $1\times10^{-4}$
\\

\midrule

HalfCheetah-v4, SAC
& $1\times10^{-3}$
& $1\times10^{-3}$
& $1\times10^{-3}$
& $1\times10^{-3}$
& $1\times10^{-3}$
& $2.5\times10^{-3}$
& $1\times10^{-3}$
& $1\times10^{-4}$
\\

Walker2d-v4, SAC
& $1\times10^{-3}$
& $1\times10^{-3}$
& $1\times10^{-3}$
& $1\times10^{-3}$
& $1\times10^{-3}$
& $2.5\times10^{-3}$
& $1\times10^{-3}$
& $1\times10^{-4}$
\\

\bottomrule
\end{tabular}
}
\end{table*}

\begin{table*}[t]
\centering
\caption{RADAR hyperparameters used in the reported experiments.}
\label{tab:Hyperparameters}
\small
\begin{tabular}{lccccc}
\toprule
Setting & $\eta$ & $\beta_1$ & $\beta_2$ & $\gamma$ & $\ell$ \\
\midrule
CIFAR-10, ViT (6.3M)
& $5\mathrm{e}{-4}$ & 0.9 & 0.999 & 0.1 & $5\mathrm{e}{-6}$ \\

CIFAR-100, ResNet-50
& $1\mathrm{e}{-3}$ & 0.9 & 0.999 & 0.1 & $1\mathrm{e}{-5}$ \\

\midrule

WikiText-103, GPT-2 Small
& $1.5\mathrm{e}{-3}$ & 0.9 & 0.999 & 0.1 & $1.5\mathrm{e}{-5}$ \\

FineWeb-Edu-10B, Qwen2.5-0.5B
& $1.5\mathrm{e}{-3}$ & 0.9 & 0.95 & 0.01 & $1.5\mathrm{e}{-5}$ \\

WikiText-2 SFT, GPT-2 Small
& $1\mathrm{e}{-4}$ & 0.9 & 0.999 & 0.1 & $1\mathrm{e}{-6}$ \\

\midrule

HalfCheetah-v4, SAC
& $1\mathrm{e}{-3}$ & 0.9 & 0.999 & 0.01 & $1\mathrm{e}{-5}$ \\

Walker2d-v4, SAC
& $1\mathrm{e}{-3}$ & 0.9 & 0.999 & 0.01 & $1\mathrm{e}{-5}$ \\

\bottomrule
\end{tabular}
\end{table*}

\subsection{Language Modeling}
\label{sec:language}

We provide implementation details and additional training curves for the GPT-2 language modeling experiments described in Section~\ref{subsec:experimental_setup}. We consider two settings: pre-training GPT-2 from scratch on WikiText-103 and fine-tuning a pretrained GPT-2 on WikiText-2. Both experiments are formulated as causal language modeling tasks. The detailed settings are summarized in Table~\ref{tab:lm_hyperparameters}.

\begin{table}[t]
\centering
\caption{Hyperparameter settings for language modeling experiments.}
\label{tab:lm_hyperparameters}
\small
\resizebox{\columnwidth}{!}{
\begin{tabular}{lccc}
\toprule
Setting 
& WikiText-103 + GPT-2 Small
& FineWeb-Edu-10B + Qwen2.5-0.5B
& WikiText-2 + GPT-2 Small\\
\midrule

Training type 
& Pre-training from scratch 
& Pre-training from scratch
& Fine-tuning \\

Training budget
& 10,000 steps
& 1B tokens (1,908 steps)
& 3 epochs \\

Sequence length / block size 
& 512 
& 2048
& 512 \\

Per-device batch size 
& 16 
& 4
& 16 \\

Gradient accumulation 
& 4 
& 16
& 4 \\

Effective batch size 
& 64 
& 256
& 64 \\

LR scheduler 
& Cosine 
& Cosine
& Cosine \\

Warmup 
& 500 steps 
& 40 steps
& 0 steps \\

Max gradient norm 
& 1.0 
& 1.0
& 1.0 \\

Precision 
& FP16 
& BF16
& FP16 \\

Logging frequency 
& Every 10 steps 
& Every 10 steps
& Every step \\

Evaluation frequency 
& Every 500 steps 
& Every 100 steps
& Every 10 steps \\

\bottomrule
\end{tabular}
}
\end{table}

\paragraph{WikiText-103 pre-training.}
We pre-train a GPT-2 model from scratch on WikiText-103. The model is initialized from the GPT-2 configuration without loading pretrained weights, while the GPT-2 tokenizer is used for tokenization. Empty lines are removed from the raw corpus, and the remaining text is tokenized, concatenated, and split into fixed-length blocks of 512 tokens. The model is trained for 10,000 update steps with an effective batch size of 64, FP16 mixed precision, a cosine learning-rate scheduler with 500 warmup steps, and gradient clipping with maximum norm 1.0.

\paragraph{FineWeb-Edu-10B pre-training.}
We further evaluate large-scale language-model pre-training by training
Qwen2.5-0.5B from scratch on a 10B-token subset of
FineWeb-Edu. The model is initialized from the Qwen2.5-0.5B configuration without loading pretrained weights, while the corresponding tokenizer is used for tokenization. The training corpus is tokenized, concatenated, and split into fixed-length blocks of 2,048 tokens. We train the model for a total budget of 1B tokens on 4$\times$ NVIDIA A100 GPUs, with a per-device batch size of 4 and 16 gradient-accumulation steps, resulting in an effective batch size of 256. Training uses BF16 mixed precision, a cosine learning-rate scheduler with 40 warmup steps, and gradient clipping with a maximum norm of 1.0. Validation loss is evaluated every 100 update steps.

\paragraph{WikiText-2 fine-tuning.}
We fine-tune a pretrained GPT-2 model on the WikiText-2 raw dataset. The input texts are tokenized using the GPT-2 tokenizer, concatenated, and split into blocks of 512 tokens. The model is trained for 3 epochs with an effective batch size of 64, FP16 mixed precision, a cosine learning-rate scheduler without warmup, and gradient clipping with maximum norm 1.0. Validation loss is evaluated every 10 steps.

\subsection{Wall-Clock Time and GPU Memory Overhead on GPT-2 Pre-Training}

\begin{table}[t]
\centering
\caption{Wall-clock time and peak GPU memory usage on GPT-2 pre-training. Results are averaged over five random seeds.}
\label{tab:runtime_memory_overhead}
\begin{tabular}{lcc}
\toprule
Optimizer
& Total Wall Time (s)
& Peak CUDA Allocated (GiB) \\
\midrule
RADAR & 4727.2583 & 13.2706 \\
AdamW & 4496.5346 & 12.8053 \\
NAdam & 4500.5386 & 12.8053 \\
\bottomrule
\end{tabular}
\end{table}

Compared with AdamW and NAdam, RADAR increases the total wall-clock time by
approximately $5.13\%$ and $5.04\%$, respectively. Its peak CUDA-allocated
memory is $13.2706$ GiB, corresponding to an additional $0.4653$ GiB, or
approximately $3.63\%$, relative to both AdamW and NAdam. These results suggest
that RADAR introduces only modest wall-clock-time and GPU-memory overhead.

% \subsection{Ablation Study}
% \label{subsec:ablation}

% We conducted ablation studies on CIFAR-10 with ViT to isolate the effects of the two RADAR-specific components: decoupled residual correction (DRC) and gradient-difference momentum filtering (GDF). The coefficient $\ell$ controls the strength of DRC, while $\gamma$ controls the strength of SF. All ablated variants use the same training setting as the CIFAR-10 ViT experiment.

% Figure~\ref{fig:ablation_study_vit} compares the entire RADAR optimizer with three variants: RADAR without SF, RADAR without DRC and RADAR without both components. The full RADAR achieves the lowest training and test losses in the late training stage, showing that the two components jointly improve optimization performance. Removing SF leads to a clear degradation in both the training and the test loss, suggesting that the gradient-difference term improves the momentum estimate and stabilizes the update direction. Removing DRC also weakens the performance, confirming that the residual correction contributes beyond the relativistic adaptive geometry alone. The variant without both components performs worst, further supporting the combined design of RADAR.

% \begin{figure*}[t]
%     \centering
%     \includegraphics[width=\linewidth]{Aba.pdf}
%     \caption{
%     Ablation study of RADAR on CIFAR-10 with ViT (6.3M). 
%     We compare the full optimizer with variants removing second-order momentum filtering (w/o SF), decoupled residual correction (w/o DRC), and both components.
%     }
%     \label{fig:ablation_study_vit}
% \end{figure*}

\subsection{Sensitivity}

We investigate the sensitivity of RADAR to its main hyperparameters. In particular, we study the interaction between the two correction coefficients, $\ell$ and $\gamma$, through a $5\times5$ grid search on WikiText-2 with GPT-2 fine-tuning. We additionally vary the relativistic-geometry parameters $\delta$ and $\zeta$ separately around their default values. The goal is to assess whether the hyperparameters used in the main experiments lie within broad and stable regions rather than near narrowly tuned task-specific optima.

\begin{figure*}[t]
\centering
\includegraphics[width=\linewidth]{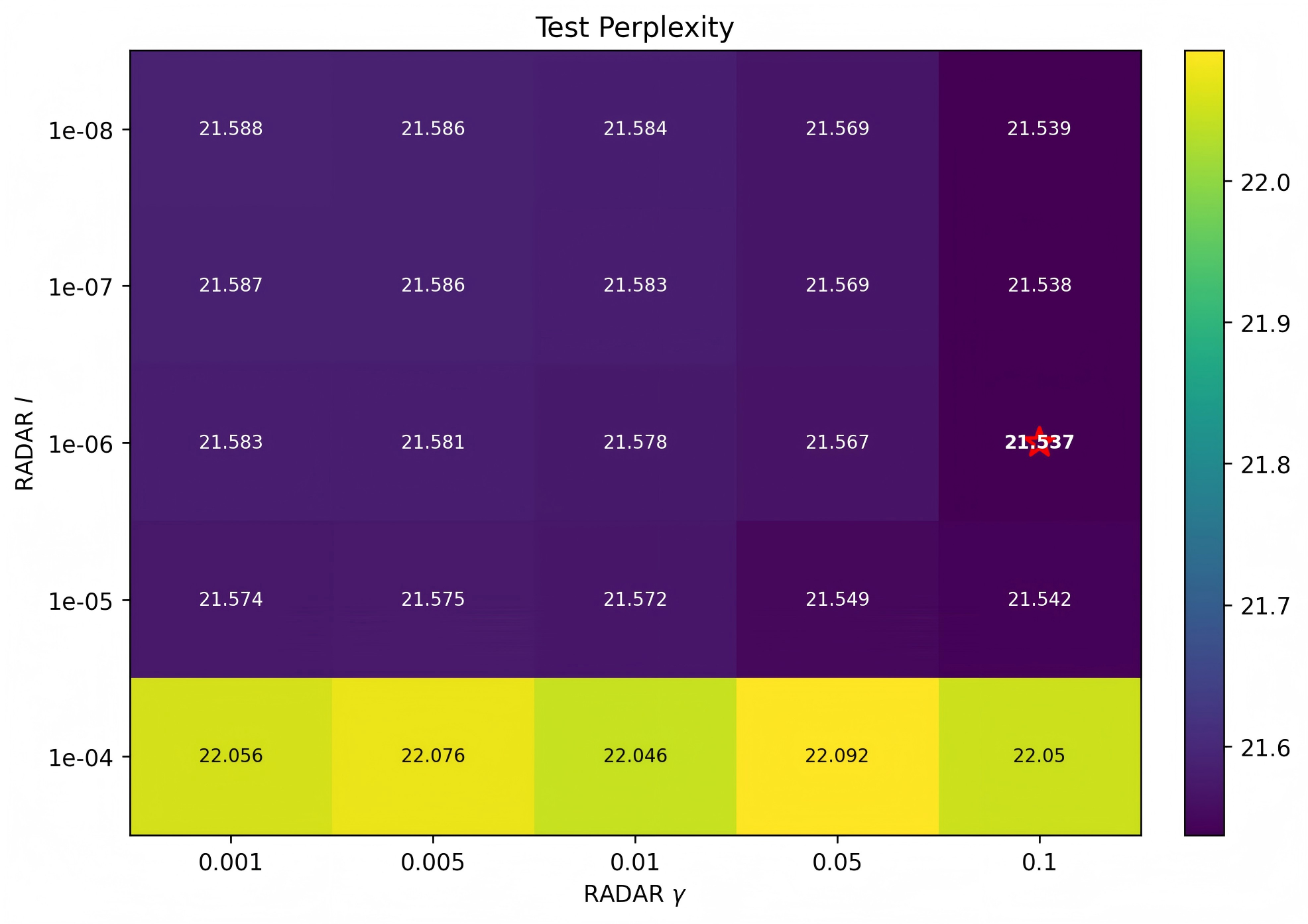}
\caption{
Joint sensitivity analysis of $\ell$ and $\gamma$ on WikiText-2 with GPT-2 fine-tuning. Each cell reports test perplexity (PPL), where lower is better. RADAR remains stable over a broad range of $\ell$ and $\gamma$, while performance degrades when $\ell$ becomes comparable to the learning-rate scale.
}
\label{fig:sensitivity}
\end{figure*}

Figure~\ref{fig:sensitivity} shows the joint sensitivity of $\ell$ and $\gamma$. For $\ell\in[10^{-8},10^{-5}]$, the test perplexity varies only moderately across a wide range of $\gamma$, with particularly stable performance for $\gamma\in[0.1,0.2]$. The best result, a test perplexity of $21.5260$, is obtained at $\ell=10^{-6}$ and $\gamma=0.2$. Importantly, the default configuration used in our experiments, $\ell=0.01\eta_0=10^{-6}$ and $\gamma=0.1$, also lies within this low-perplexity region rather than at an isolated optimum. Increasing $\ell$ to $10^{-4}$, which is on the same scale as the learning rate $\eta_0$, leads to a clear degradation across all values of $\gamma$. This suggests that the residual correction is robust when kept sufficiently smaller than the primary optimization step, while overly large correction magnitudes can adversely affect training.

We observe similar robustness with respect to the relativistic-geometry parameters. Varying $\delta$ by $\pm 10\%$ around its default value of $1$ changes the test perplexity by at most $0.0595$, corresponding to less than $0.28\%$, and the method remains stable over the broader range $\delta\in[0.8,1.2]$. Likewise, varying $\zeta$ from $10^{-18}$ to $10^{-14}$ changes the test perplexity by less than $0.01$, or below $0.05\%$, indicating that performance is largely insensitive to the default choice $\zeta=10^{-16}$.

Overall, these results indicate that RADAR does not rely on narrowly tuned hyperparameters. The default choices of $\ell$, $\gamma$, $\delta$, and $\zeta$ all lie within stable performance regions and are therefore fixed, rather than retuned for individual random seeds, in the five-seed main experiments.

\end{document}